%% file: paper.tex
\documentclass[twocolumn,nospthms]{svjour3} 
\include{preamble}
\usepackage{microtype}
\begin{document}

\title{Level-Based Analysis of the
Univariate Marginal Distribution Algorithm  
\thanks{Preliminary versions of this work appeared 
in the Proceedings of the 2015 and 2017 Genetic and 
Evolutionary Computation Conference (GECCO 2015 \& 2017)} 
}
\author{Duc-Cuong~Dang \and                                                     %
        Per~Kristian~Lehre \and 
        Phan~Trung~Hai~Nguyen}
        
\institute{D.-C. Dang
\at Independent researcher \\
\email{duc-cuong.dang@hds.utc.fr}
\and
P.K. Lehre \and P.T.H. Nguyen
\at School of Computer Science, 
University of Birmingham\\
Birmingham B15 2TT, 
United Kingdom\\
\email{\{p.k.lehre, p.nguyen\}@cs.bham.ac.uk}
}

\titlerunning{Level-Based Analysis of the Univariate Marginal Distribution Algorithm} 
\authorrunning{Duc-Cuong Dang, Per Kristian Lehre \& Phan Trung Hai Nguyen} 


\maketitle


\begin{abstract}
  Estimation of Distribution Algorithms 
  (\edas) are stochastic
  heuristics that search for optimal solutions by learning and
  sampling from probabilistic models. Despite their popularity in
  real-world applications, there is little rigorous understanding of
  their performance. Even for the Univariate Marginal Distribution
  Algorithm (UMDA) -- a simple population-based \eda assuming independence between
  decision variables -- the optimisation time on the linear
  problem \onemax was until recently undetermined. The incomplete
  theoretical understanding of \edas is mainly due to lack of
  appropriate analytical tools.

  We show that the recently developed \emph{level-based theorem} for
  non-elitist populations combined with anti-concentration results
  yield upper bounds on the expected optimisation time of the \umda.
  This approach results in the bound $\bigO{n\lambda\log \lambda+n^2}$
  on two problems, \leadingones and \binval, for population sizes
  $\lambda>\mu=\Omega(\log n)$, where
  $\mu$ and $\lambda$ are parameters of the algorithm. We also prove that the
  \umda with population sizes
  $\mu\in \bigO{\sqrt{n}} \cap \Omega(\log n)$ optimises \onemax in
  expected time $\bigO{\lambda n}$, and for larger population
  sizes $\mu=\Omega(\sqrt{n}\log n)$, in expected time
  $\bigO{\lambda\sqrt{n}}$. The facility and
  generality of our arguments suggest that this is a promising
  approach to derive bounds on the expected optimisation time of \edas.
\end{abstract}


          

\section{Introduction}\label{sec:intro}

Estimation of Distribution Algorithms (\edas) are a class 
of randomised search heuristics with many practical applications
\cite{Ducheyne2005,bib:Gu2015,KOLLAT2008828,Yu2006,Zinchenko2002}. 
Unlike traditional Evolutionary Algorithms (\eas)
which search for 
optimal solutions using genetic operators such as mutation or crossover, 
\edas build and maintain a probability distribution of the current 
population over the search space, 
from which the next generation of individuals 
is sampled.  Several EDAs have been developed over the last
decades. The algorithms differ in how
they capture interactions among decision variables,
as well as in how they build and update their probabilistic models.
\edas are often classified as either \textit{univariate} or
\textit{multivariate}; the former 
treat each variable independently, while the latter also consider
variable dependencies \cite{bib:Shapiro2005}. 
Well-known univariate
\edas include the compact Genetic Algorithm (\cga \cite{bib:Harik}),
the Population-Based  Incremental  Learning  Algorithm (\pbil 
\cite{bib:Baluja1994}),
and the Univariate Marginal Distribution Algorithm
(\umda \cite{bib:Muhlenbein1996}). 
Given a problem instance size $n$, univariate \edas 
represent probabilistic models as an $n$-vector, where each vector component 
is called a \textit{marginal}.
Some Ant Colony Optimisation (ACO) algorithms and even certain
single-individual \eas can be 
cast in the same framework as univariate
\edas (or $n$-$\Ber$-$\lambda$-\eda, see, e.g.,
\cite{bib:Friedrich2016,bib:Sudholt2016,bib:Hauschild,Krejca2018survey}).
Multivariate \edas, such as the
Bayesian Optimisation Algorithm,  which builds a 
Bayesian network with nodes and edges representing 
variables and conditional dependencies respectively, 
attempt to learn relationships between 
decision variables  \cite{bib:Hauschild}. 
The surveys \cite{bib:Armananzas2008,bib:Hauschild,Santana2016}
describe further variants and applications of \edas.

Recently \edas have drawn a growing attention from the 
theory community of evolutionary computation 
\cite{bib:Dang2015a,bib:Friedrich2016,Lehre2017,bib:Witt2017,bib:Wu2017,bib:Krejca,Witt2018Domino,NguyenPbil2018,DoerrSigEDA2018,bib:Lengler2018}. The aim of the theoretical analyses of \edas 
in general is to gain insights 
into the behaviour of the algorithms when optimising an objective function, 
especially in terms of the  optimisation time, that is 
the number of function evaluations, required 
by the algorithm until an optimal solution has been found for the 
first time.
Droste \cite{bib:Droste2006} provided the first rigorous runtime
analysis of an \eda, specifically the 
\cga.
Introduced in \cite{bib:Harik}, 
the \cga samples two individuals in each generation and updates the
probabilistic model according to the fittest of these individuals. A quantity of $\pm 1/K$   
is added to the marginals for each bit position where the two 
individuals differ. 
The reciprocal $K$ of this quantity is often referred to as the abstract
\emph{population size} of a genetic algorithm that the \cga is supposed to model.
Droste  showed
a lower bound  $\Omega(K\sqrt{n})$ on the expected optimisation time of
the \cga for any pseudo-Boolean function \cite{bib:Droste2006}.
He also proved the upper bound $\mathcal{O}(nK)$ for any 
linear function, where $K=n^{1/2+\varepsilon}$ 
for any small constant $\varepsilon>0$. 
Note that each marginal 
of the \cga considered in \cite{bib:Droste2006} 
is allowed to reach the extreme values zero and one. Such an
algorithm is referred to as 
an \eda \emph{without margins}, since in contrast it is 
possible to reinforce some margins (also called \emph{borders})
on the range of values for each marginal to keep it 
away from the extreme probabilities, often
within the interval $[1/n,1-1/n]$. An \eda without margins can 
prematurely converge to a sub-optimal solution; thus, the runtime bounds 
of \cite{bib:Droste2006} were in fact conditioned on the event that 
early convergence never happens.
Very recently, 
Witt~\cite{Witt2018Domino} 
 studied an effect called \textit{domino convergence} on \edas,
where bits with heavy weights
tend to be optimised before bits with light weights. By deriving a
lower bound of $\Omega(n^2)$ 
on the expected optimisation time of the \cga on 
\binval for any value of $K>0$, 
Witt confirmed the claim made  earlier  by 
Droste \cite{bib:Droste2006} that \binval is a harder problem for the \cga
than the \onemax problem. Moreover,
Lengler et al. \cite{bib:Lengler2018} considered 
$K=\bigO{\sqrt{n}/\log^2 n}$, 
which was not covered by Droste in \cite{bib:Droste2006}, 
and obtained
a lower bound of $\Omega(K^{1/3}n+n\log n)$ 
on the expected optimisation time
of the \cga on \onemax. Note that if 
$K=\Theta(\sqrt{n}/\log^2 n)$, the above lower bound will be 
$\Omega(n^{7/6}/\log^2 n)$, which further tightens the bounds 
on the expected optimisation time of the \cga.


An algorithm closely related to the \cga with (reinforced) margins is 
the $2$-Max Min Ant System with iteration best (\tmmasib). The two
algorithms differ only
slightly  in the update procedure of the model, and \tmmasib is 
parameterised by an evaporation factor
$\rho\in(0,1)$. 
Sudholt~and~Witt~\cite{bib:Sudholt2016} proved the lower bounds 
$\Omega(K\sqrt{n}+n\log{n})$ and $\Omega(\sqrt{n}/\rho+n\log{n})$
for the two algorithms on \onemax under any setting, and upper bounds 
$\mathcal{O}(K\sqrt{n})$ and $\mathcal{O}(\sqrt{n}/\rho)$ when $K$ and 
$\rho$ are in $\Omega(\sqrt{n}\log{n})$. Thus, the optimal expected 
optimisation time $\Theta(n\log{n})$ of the \cga and the 
\tmmasib on \onemax is 
achieved by setting these parameters to $\Theta(\sqrt{n}\log{n})$. 
The analyses revealed that choosing lower parameter values result in
strong fluctuations that may cause  many marginals
(or \textit{pheromones} in the context of ACO) to fix early at
the lower margin, which then need
to be repaired later. On the other hand,
choosing higher parameter values resolve the issue
but may slow down the learning process.

Friedrich et al. \cite{bib:Friedrich2016} pointed out two 
behavioural properties of 
univariate \edas at each bit position: a 
\emph{balanced} \eda would be sensitive to signals in the fitness, 
while a \emph{stable} one would remain 
uncommitted under a biasless fitness function. 
During the optimisation of \leadingones, 
when some bit positions are 
temporarily neutral, while the others are not, both properties appear
useful to avoid commitment to wrong decisions. Unfortunately, 
many univariate 
\edas without margins, including 
the \cga, the \umda, the \pbil and some 
related algorithms  are balanced but not 
stable \cite{bib:Friedrich2016}. 
A more stable version of the \cga \---
the so-called stable \cga (or \scga) -- was 
then introduced in \cite{bib:Friedrich2016}. 
Under appropriate settings, it yields an expected optimisation time 
 of $\mathcal{O}(n\log n)$ on \leadingones 
 with a probability polynomially 
close to one. 
Furthermore, a recent study by Friedrich et al. \cite{bib:Friedrich2017}
showed that \cga can cope with 
higher levels of noise more efficiently than
mutation-only heuristics do.

Introduced by Baluja \cite{bib:Baluja1994}, 
the \pbil is another univariate \eda. Unlike the \cga that 
samples two solutions in each generation, the \pbil samples a population
of $\lambda$ individuals, from which the $\mu$ 
fittest individuals
are selected to update 
the probabilistic model, i.e., \emph{truncation selection}. The new probabilistic model
is obtained using a convex 
combination with a \textit{smoothing parameter} 
$\rho\in (0,1]$ of the current model and the frequencies 
of ones among all selected
individuals at that bit position.
The \pbil can be seen as a special case of the 
\emph{cross-entropy method} \cite{bib:Rubinstein2004} 
on the binary hypercube $\{0,1\}^n$. 
Wu et al. \cite{bib:Wu2017} analysed the runtime of the 
\pbil on \onemax and \leadingones. The 
authors argued that due to the use of a sufficiently large
population size, it is possible to
prevent the marginals from reaching the lower border
early even when a large smoothing parameter $\rho$ is used. 
Runtime results were proved for the \pbil without margins on \onemax and the
\pbil with margins on \leadingones, and were
then compared to the runtime of some
Ant System approaches. However, the required 
population size is large, 
\ie $\lambda=\omega(n)$. Very recently,
Lehre and Nguyen~\cite{NguyenPbil2018} obtained an upper bound of 
$\mathcal{O}(n\lambda\log \lambda+n^2)$ on 
the expected optimisation time 
for the \pbil with margins 
 on \binval and \leadingones, which improves
 the previously known upper bound
 in \cite{bib:Wu2017}  
by a factor of $n^{\varepsilon}$, 
where $\varepsilon$ is some positive constant, 
for smaller population 
sizes $\lambda=\Omega(\log n)$. 
%

The \umda is a special case of the \pbil with the largest 
smoothing parameter $\rho = 1$, that is, the 
probabilistic model for the next generation 
 depends solely on the selected individuals 
in the current  population. This characteristic distinguishes  
the \umda from the \cga and \pbil in general. 
The algorithm has 
a wide range of applications, not only 
in computer science, but also in 
other areas like population genetics and 
bioinformatics \cite{bib:Gu2015,Zinchenko2002}.
Moreover, the \umda is related to the notion of \textit{linkage
equilibrium} \cite{bib:Slatkin2008,bib:muhlenbein2002}, 
which is a popular model assumption in 
population genetics. Thus, studies of the \umda can contribute to the 
understanding of population dynamics in population genetics.

Despite an increasing momentum in the runtime analysis of \edas 
over the last few years,
our understanding of the \umda 
in terms of runtime is still limited. 
The algorithm was early
analysed in a series of papers 
\cite{bib:Chen2009a,bib:Chen2007,bib:Chen2009b,bib:Chen2010}, 
where time-complexities of the \umda 
on simple uni-modal functions were derived. 
These results showed
that the \umda with margins often outperforms 
the \umda without margins, 
especially on functions like \bvleadingones, 
which is a uni-modal problem.
The possible reason behind the failure of the \umda without margins
is due to fixation, causing no further progression for 
the corresponding decision variables. 
The \umda with margins is able 
to avoid this by ensuring that each 
search point always has a positive chance to be sampled.
Shapiro  investigated the \umda with a 
different selection mechanism than 
truncation selection \cite{bib:Shapiro2005}. In
particular, this variant of the \umda selects individuals whose
fitnesses are no less than the mean fitness of all individuals in the 
current population when
updating the probabilistic model. 
By representing the \umda as 
a Markov chain, the paper showed that 
the population size has to be at least $\sqrt{n}$ 
for the \umda to prevent the probabilistic model
from quickly converging to the corner of the hypercube on \onemax. 
This phenomenon is well-known as \textit{genetic drift} \cite{GeneticDrift}.
%
A decade later, the first upper 
bound on the expected optimisation time of the \umda on 
\onemax was revealed \cite{bib:Dang2015a}. 
Working on the standard \umda 
using truncation selection,
Dang~and~Lehre~\cite{bib:Dang2015a} proved an upper 
bound of $\mathcal{O}(n\lambda\log \lambda)$ 
on the expected optimisation time 
of the \umda on \onemax, 
assuming a population size $\lambda=\Omega(\log n)$.
If $\lambda = \Theta(\log n)$, then 
the upper bound is $\mathcal{O}(n\log n\log \log n)$.
%
Inspired by the previous work of \cite{bib:Sudholt2016} 
on \cga/\tmmasib, Krejca~and~Witt \cite{bib:Krejca} obtained 
a lower bound of $\Omega(\mu\sqrt{n}+n\log n)$ for the \umda on 
\onemax via  \textit{drift analysis}, where 
$\lambda = (1+\Theta(1))\mu$.
Compared to \cite{bib:Sudholt2016}, the analysis is much more 
involved since, unlike in \cga/\tmmasib where each change of marginals
between consecutive generations is small and limited by to the 
smoothing parameter, large changes are always possible in the \umda.
From these results, we observe that the latest upper and lower 
bounds for the \umda on \onemax  still differ by 
$\Theta(\log \log n)$. This
raises the question of whether this gap could be 
closed.

\begin{table*}[!ht]
\begin{threeparttable}
	\centering
		\caption{Expected optimisation time (number of fitness evaluations)
        of univariate \edas on the three problems \onemax, 
        \leadingones
        and \binval.}
        \label{tab:summary-of-runtime}
		\begin{tabular}{@{}p{3cm} p{3cm} p{6cm} p{4cm}@{}}
			\toprule
			Problem& Algorithm & Constraints & Runtime\\
			\midrule 
			\onemax & \umda & $\lambda=\Theta(\mu), \lambda=\bigO{\text{poly(n)}}$& $\Omega(\lambda\sqrt{n}+n\log n)$ \cite{bib:Krejca}\\
			
			& &$\lambda=\Theta(\mu),~ \mu=\Omega(\log n)\cap o(n)$& $\bigO{\lambda n}$ \cite{bib:Witt2017}\\
			
			&&$\lambda=\Theta(\mu), ~ \mu= \Omega(\sqrt{n}\log n)$&$\bigO{\lambda\sqrt{n}}$ \cite{bib:Witt2017}\\
			
			&&$\lambda=\Omega(\mu), ~\mu = \Omega(\log n) \cap \bigO{\sqrt{n}}$& $\bigO{\lambda n}$ [Thm.~\ref{onemax-small-mu}] \\
			&&$\lambda=\Omega(\mu), ~\mu =\Omega(\sqrt{n}\log n)$& $\bigO{\lambda \sqrt{n}}$ [Thm.~\ref{onemax-large-mu}] \\
            &\pbil\tnote{*} &$\mu=\omega(n), 
            \lambda=\omega(\mu)$&$\omega(n^{3/2})$ \cite{bib:Wu2017}\\
			&\cga&$K=n^{1/2+\epsilon}$&$\Theta(K\sqrt{n})$ \cite{bib:Droste2006}\\	
            &&$K=\bigO{\sqrt{n}/\log^2 n}$&$\Omega(K^{1/3}n+n\log n)$ \cite{bib:Lengler2018}\\
            
			&$\textsc{scGA}$&$\rho=\Omega(1/\log n), a=\Theta(\rho), c>0$&
			$\Omega(\min\{2^{\Theta(n)},2^{c/\rho}\})$ \cite{DoerrSigEDA2018}\\
			\midrule
			
			\leadingones& \umda & $\mu= \Omega(\log n), \lambda=\Omega(\mu)$ &$\bigO{n\lambda \log \lambda + n^2}$ [Thm.~\ref{thm:leadingones}]\\
            
            &\pbil& $\lambda=n^{1+\epsilon}, \mu= \bigO{n^{\epsilon/2}}, \epsilon\in (0,1)$&$\bigO{n^{2+\epsilon}}$ \cite{bib:Wu2017}\\
			&& $\lambda= \Omega(\mu), \mu=\Omega(\log n)$&$\bigO{n\lambda\log \lambda+n^2}$ \cite{NguyenPbil2018}\\
           		
			&$\textsc{scGA}$&$\rho=\Theta(1/\log n), a=\bigO{\rho}$&
			$\bigO{n\log n}$ \cite{bib:Friedrich2016}\\
			
			\midrule
			
			\binval&\umda& $\mu = \Omega(\log n), \lambda=\Omega(\mu)$&$\bigO{n\lambda\log \lambda+n^2}$ [Thm.~\ref{thm:leadingones}]\\
            
            &\pbil& $\lambda= \Omega(\mu), \mu=\Omega(\log n)$ & $\bigO{n\lambda\log \lambda+n^2} $ \cite{NguyenPbil2018}\\          
			
			&\cga&$K=n^{1/2+\epsilon}$&$\Theta(Kn)$ \cite{bib:Droste2006}\\
			&&$K>0$&$\Omega(n^2)$ \cite{Witt2018Domino}\\
				\bottomrule
		\end{tabular}
        \begin{tablenotes}\footnotesize
\item[*] without margins
\end{tablenotes}
\end{threeparttable}
	\end{table*}

This paper derives upper bounds on the expected optimisation time 
of the \umda on the following 
problems: \onemax, \binval, and \leadingones. 
The preliminary versions of this work appeared in \cite{bib:Dang2015a}
and \cite{Lehre2017}.
Here we use the improved version of the \emph{level-based 
analysis} technique \cite{bib:Corus2017}. The analyses for 
\leadingones and \binval are straightforward and similar to 
each other, \ie yielding the same runtime $\mathcal{O}(n\lambda\ln\lambda+n^2)$;
hence, they will serve the purpose of introducing the technique
in the context of \edas. 
Particularly, we only require population sizes $\lambda = \Omega(\log{n})$ 
for \leadingones
which is much smaller than previously thought \cite{bib:Chen2007,bib:Chen2009b,bib:Chen2010}.
For \onemax, we give a more detailed analysis 
so that an expected optimisation time 
of $\mathcal{O}(n\log n)$ is derived if 
the population size is chosen appropriately. This 
significantly improves the results in 
\cite{bib:Corus2017,bib:Dang2015a}
and matches the recent lower bound of \cite{bib:Krejca} as well as the performance 
of the (1+1)~\ea. More specifically, we assume $\lambda \geq b \mu$
for a sufficiently large constant $b>0$,
and separate two regimes of small and large selected populations:
the upper bound $\mathcal{O}(\lambda n)$ 
is derived for $\mu = \Omega(\log n) \cap \mathcal{O}(\sqrt{n})$,
and the upper bound $\mathcal{O}(\lambda\sqrt{n})$ 
is shown for $\mu = \Omega(\sqrt{n}\log n)$. These results 
exhibit the applicability of the
level-based technique in the runtime
analysis of (univariate) \edas. 
Table~\ref{tab:summary-of-runtime} summarises
the latest results about the runtime
analyses of univariate \edas on simple benchmark problems; 
see \cite{Krejca2018survey} for a recent survey on the theory of
\edas.

\textit{Related independent work:}
Witt~\cite{bib:Witt2017} 
independently obtained the upper bounds of 
$\mathcal{O}(\lambda n)$ and $\mathcal{O}(\lambda\sqrt{n})$ 
on the expected optimisation 
time of the \umda  on \onemax for $\mu = \Omega(\log n)\cap o(n)$
and $\mu = \Omega(\sqrt{n}\log n)$, 
respectively, and $\lambda=\Theta(\mu)$ using an involved drift analysis. 
While our results do not hold for 
$\mu =\Omega(\sqrt{n})\cap \bigO{\sqrt{n}\log n}$,
our methods yield significantly easier proofs. 
Furthermore, our analysis also holds when the parent 
population size $\mu$ is not proportional 
to the offspring population size $\lambda$, 
which is not covered in \cite{bib:Witt2017}.

This paper is structured as follows. 
Section~\ref{sec:umda-algorithm} introduces the notation 
used throughout the paper and 
the \umda  with margins. We also introduce the techniques 
used, including the 
level-based theorem, which is central in the paper, and 
an important sharp bound on the sum of 
Bernoulli random variables. 
Given all necessary tools, Section~\ref{sec:umda-leadingones} presents 
upper bounds on the expected optimisation 
time of the \umda  on both \leadingones
and \binval, followed by
the derivation of the upper bounds on the 
expected optimisation time of the \umda 
on \onemax. The latter consists of two smaller 
subsections according to two different ranges of
values of the parent population size. 
Section~\ref{sec:empirical} presents 
a brief empirical analysis of the \umda  
on \leadingones, \binval and \onemax to support 
the theoretical findings in 
Sections~\ref{sec:umda-leadingones} and \ref{sec:umda-onemax}.
Finally, our concluding remarks are 
given in Section~\ref{sec:conclusion}.




\section{Preliminaries}\label{sec:umda-algorithm}

This section describes the three standard benchmark problems, 
the algorithm under investigation
and the level-based theorem, which is 
a general method to derive upper bounds on the 
expected optimisation time of non-elitist 
population-based algorithms. 
Furthermore, a sharp upper bound on  
the sum of independent Bernoulli trials, 
which is essential in the runtime analysis of the \umda on 
\onemax for a small population size, is presented, followed by 
Feige's inequality.

We use the following notation throughout the paper. 
The natural logarithm is denoted as $\ln(\cdot)$, and
$\log(\cdot)$ denotes the logarithm with base 2. 
Let $[n]$ be the set $\{1,2,\dots,n\}$.
The floor and ceiling functions are $\lfloor x\rfloor$ 
and $\lceil x\rceil$, respectively, for $x \in \mathbb{R}$.
For two random variables $X, Y$, we use $X \preceq Y$ to indicate that $Y$ 
stochastically dominates $X$, that is 
$\prob{X \geq k} \leq \prob{Y \geq k}$ for 
all $k \in \mathbb{R}$. 

We often consider a partition of the \textit{finite} 
search space $\mathcal{X}=\{0,1\}^n$ into $m$  
ordered subsets $A_1,\dots,A_m$ called \emph{levels}, \ie $A_i \cap A_j = 
\emptyset$ for any $i \neq j$ and $\cup_{i=1}^{m}A_i = \mathcal{X}$. 
The union of all levels above $j$ inclusive is denoted
 $A_{\geq j}:=\cup_{i=j}^m A_i$. 
An optimisation problem on $\mathcal{X}$ is assumed, without loss of 
generality, to be the maximisation of some  function 
$f\colon \mathcal{X} \rightarrow \mathbb{R}$.
A partition is called \emph{fitness-based} (or $f$-based) 
if for any $j \in [m-1]$ and all $x \in A_{j}$, 
$y \in A_{j+1} \colon f(y)>f(x)$. 
An $f$-based partitioning is 
called \textit{canonical} when $x,y \in A_j$ if and only if $f(x)
=f(y)$. 

Given the search space $\mathcal{X}$, each $x\in \mathcal{X}$
is called a \textit{search point} (or \textit{individual}), and 
a \textit{population} is a vector of 
search points, i.e. $P \in \mathcal{X}^{\lambda}$. 
For a finite population $P= 
\left(x^{(1)},\ldots,x^{(\lambda)}\right)$, we define $|P \cap A_j| := |\{i \in 
[\lambda] \mid x^{(i)} \in A_j\}|$, \ie the number of individuals in 
population $P$ which are in level $A_j$. 
\emph{Truncation selection}, denoted as 
\emph{$(\mu,\lambda)$-selection} for some 
$\mu<\lambda$, applied to population 
$P$ transforms it into a vector $P'$ (called 
\emph{selected} population) with $|P'|=\mu$ by discarding the $\lambda - \mu$ 
worst search points of $P$ with respect to some fitness function $f$,
were ties are broken uniformly at random.

\subsection{Three Problems}

We consider the three pseudo-Boolean functions: \onemax,
\leadingones and \binval, which are
defined over the finite binary search space 
$\mathcal{X}=\{0,1\}^n$ and
widely used  as
theoretical benchmark problems in runtime
analyses of \edas 
\cite{bib:Droste2006,bib:Dang2015a,NguyenPbil2018,bib:Krejca,bib:Witt2017,bib:Wu2017}.
Note in particular that
these problems are only required to describe and compare 
the behaviour of the \edas on problems with well-understood structures.
The first problem, as its name may suggest, 
simply counts the number of ones in the bitstring
and is widely used to test the performance of
\edas as a hill climber \cite{Krejca2018survey}. 
While the bits in \onemax have the same contributions to the overall fitness, 
\binval, which aims at maximising the 
binary value of the bitstring,
has exponentially scaled weights relative to
bit positions. In contrast, \leadingones counts the number of leading ones in the bitstring. 
Since bits in this particular 
problem are highly correlated, it is often used 
to study the ability of \edas to cope 
with dependencies among decision variables
\cite{Krejca2018survey}.

The global optimum for all functions are the all-ones bitstring, i.e. $1^n$.
For any bitstring $x=(x_1,\ldots,x_n) \in \mathcal{X}$, these functions 
are defined as follows:
\begin{definition} \label{def:onemax}
$\onemax(x) := \sum_{i=1}^{n}x_i$.
\end{definition}
\begin{definition}\label{def:leadingones}
$\leadingones(x) := \sum_{i=1}^{n}\prod_{j=1}^{i}x_j$.
\end{definition}
\begin{definition}\label{def:binval}
$\binval(x) := \sum_{i=1}^{n}2^{n-i}x_i$.
\end{definition}

\subsection{Univariate Marginal Distribution Algorithm}
Introduced  by M{\"u}hlenbein
and Paa{\ss} \cite{bib:Muhlenbein1996}, the Univariate 
Marginal Distribution Algorithm (\umda; see Algorithm~\ref{umda-algor})
is one of the simplest \edas, which assume independence between 
decision variables.
To optimise a pseudo-Boolean function
$f\colon \{0,1\}^n\rightarrow\mathbb{R}$,
the algorithm follows an iterative process: 
\textit{sample} independently and identically a population of $\lambda$ offspring from the current
probabilistic model and 
\textit{update} the model using the $\mu$ fittest 
individuals in the current population.
Each sample-and-update cycle 
is called a \textit{generation} (or \textit{iteration}).
The probabilistic model in generation $t\in \mathbb{N}$
is represented as a vector
$p_t=\left(p_t(1),\ldots,p_t(n)\right)\in [0,1]^n$, where 
each component (or \textit{marginal}) $p_t(i)\in[0,1]$ 
for $i\in[n]$ and $t\in \mathbb{N}$ is
the probability of sampling a one at the $i$-th bit position of an
offspring in generation $t$.
Each individual $x=(x_1,\ldots,x_n)\in \{0,1\}^n$ is therefore
sampled from the joint probability distribution

\begin{equation}\label{eq:prob-distribution}
\Pr\left(x \mid p_t\right)
=\prod_{i=1}^{n}p_t(i)^{x_i}\left(1-p_t(i)\right)^{(1-x_i)}.
\end{equation}

Note that the probabilistic model is initialised as
$p_0(i)\coloneqq 1/2$ for each $i\in [n]$.
Let $x_t^{(1)},\ldots,x_t^{(\lambda)}$ 
be $\lambda$ individuals that are 
sampled from the probability 
distribution (\ref{eq:prob-distribution}), then $\mu$ of which with
the fittest fitness 
are selected to obtain the next
model $p_{t+1}$. Let $x_{t,i}^{(k)}$ 
denote the value of the $i$-th bit 
position of the $k$-th individual in the current sorted population $P_t$. For
each $i \in [n]$, the corresponding marginal of the next model is
\begin{displaymath}
p_{t+1}(i)\coloneqq \frac{1}{\mu}\sum_{k=1}^{\mu}x_{t,i}^{(k)},
\end{displaymath}
which can be interpreted as the frequency of ones among the $\mu$ fittest 
individuals at bit-position $i$.

The extreme probabilities -- zero and one --
must be avoided for each marginal $p_t(i)$;
otherwise, the bit in position $i$ would remain fixed 
forever at either zero or one, obstructing
some regions of the search space. To avoid this, all marginals
$p_{t+1}(i)$ are usually 
restricted within the closed interval 
$[\frac{1}{n},1-\frac{1}{n}]$, and
such values $\frac{1}{n}$ and 
$1-\frac{1}{n}$ are called 
\textit{lower} and \textit{upper borders}, respectively. The 
algorithm in this case is known as the \umda \textit{with margins}.

\begin{algorithm}[t]
	\DontPrintSemicolon	
    \SetKwInOut{Input}{input}
    \SetKwInOut{Param}{parameter}
    \Param{offspring population size $\lambda$, 
    		parent population size $\mu$, 
            maximising $f$}
    	$t\leftarrow 0$\;
        initialise $p_0(i) \leftarrow 1/2$ for each $i \in [n]$\;
		 \Repeat{termination condition is fulfilled}{
			\For{$k=1,2,\ldots,\lambda$}{
				sample $x_{t,i}^{(k)} \sim \Ber(p_t(i))$ for each $i\in [n]$
               \;
			}		
            sort $P_t\leftarrow \{x_t^{(1)},x_t^{(2)},\ldots,x_t^{(\lambda)}\}$ s.t.
            $f(x_t^{(1)})\ge f(x_t^{(2)})\ge \ldots\ge f(x_t^{(\lambda)})$\;
			\For{$i=1,2,\ldots,n$}{
				$X_i\leftarrow \sum_{k=1}^{\mu} x_{t,i}^{(k)}$\;
                $p_{t+1}(i)\leftarrow \max\big\{\frac{1}{n}, \min\big\{1-\frac{1}{n}, \frac{X_i}{\mu}\big\}\big\}$\;	
			}
            $t \leftarrow t+1$\;
            }
	\caption{\umda with margins \label{umda-algor}}
\end{algorithm}


\subsection{Level-Based Theorem}\label{level-based-theorem-section}

We are interested in the optimisation time
of the \umda, which is a non-elitist algorithm; thus,
tools for analysing runtime for this class of algorithms are of 
importance. Currently in the literature, 
\textit{drift theorems} 
have often been used to derive
upper and lower bounds on the expected 
optimisation time of the \umda, see, e.g., \cite{bib:Witt2017,bib:Krejca} 
because they allow us to examine the dynamics
of each marginal in the vector-based probabilistic model.
In this paper, we take another perspective where we consider the 
population of individuals. To do this, we make use of the
so-called level-based theorem, which has been previously used 
to derive the first upper bound of $\bigO{n\lambda\log \lambda}$ on the 
expected optimisation time of the \umda on \onemax \cite{bib:Dang2015a}.

Introduced by Corus et al. \cite{bib:Corus2017},
the level-based theorem is a general tool that provides 
upper bounds on the expected optimisation time of 
many non-elitist population-based algorithms on a wide range of 
optimisation problems \cite{bib:Corus2017}. It has been 
applied to analyse the expected optimisation time of 
 Genetic Algorithms with or without crossover 
on various pseudo-Boolean functions and 
combinatorial optimisation problems
\cite{bib:Corus2017}, self-adaptive EAs \cite{DangLehre2016SelfAdaptation}, 
the \umda with margins on \onemax and 
\leadingones \cite{bib:Dang2015a}, and very recently
the \pbil with margins on \leadingones and \binval 
\cite{NguyenPbil2018}. 

The theorem assumes that the algorithm to be analysed can be described
in the form of Algorithm \ref{abstract-algor}.  The population $P_t$
at generation $t\in \mathbb{N}$ of $\lambda$ individuals is
represented as a vector
$(P_t(1),\ldots,P_t(\lambda))\in \mathcal{X}^\lambda$.  The theorem is
general because it does not assume specific fitness functions,
selection mechanisms, or generic operators like mutation and
crossover. Rather, the theorem assumes that there exists, possibly
implicitly, a mapping $\mathcal{D}$ from the set of populations
$\mathcal{X}^\lambda$ to the space of probability distributions over
the search space $\mathcal{X}$. The distribution $\mathcal{D}(P_t)$
depends on the current population $P_t$, and all individuals in
population $P_{t+1}$ are sampled identically and independently from
this distribution~\cite{bib:Corus2017}. The assumption of independent
sampling of the individual holds for the UMDA, and many other
algorithms.

\begin{algorithm}[t]
	\DontPrintSemicolon
    $t \leftarrow 0$\;  
    initialise population $P_0$\;
    \Repeat{termination condition is fulfilled}{
      \For{$i=1,\ldots,\lambda$}{
				sample $P_{t+1}(i) \sim \mathcal{D}(P_t)$ independently\;
              }
                $t \leftarrow t+1$\;
			}
	\caption{Non-elitist population-based algorithm}
    \label{abstract-algor}
\end{algorithm}

The theorem assumes a partition
$A_1,\ldots,A_m$ of the finite search space $\mathcal{X}$
into $m$ subsets, which we call \textit{levels}. 
We assume that the last level $A_m$ 
consists of all optimal solutions. 
Given a partition of the search space $\mathcal{X}$, 
we can state the level-based theorem as follows:

\begin{theorem}[\cite{bib:Corus2017}]\label{thm:levelbasedtheorem}
	Given a partition $(A_1,\ldots,A_m)$ 
    of $\mathcal{X}$, define 
	\begin{math}
		T\coloneqq\min\{t\lambda \mid |P_t\cap A_m|>0\},
              \end{math}
              where for all $t\in\mathbb{N}$,
              $P_t\in\mathcal{X}^\lambda$ is the population of
              Algorithm~\ref{abstract-algor} in generation $t$.
        If there exist 
	$z_1,\ldots,z_{m-1}, \delta \in (0,1]$, and $\gamma_0\in (0,1)$ 
	such that for any population $P_t \in \mathcal{X}^\lambda$,
	\begin{itemize}
		\item \textbf{(G1)} for each level $j\in[m-1]$, if $|P_t\cap A_{\geq j}|\geq \gamma_0\lambda$ then 
			\begin{displaymath}
							\Pr\nolimits_{y \sim \mathcal{D}(P_t)}\left(y \in A_{\geq j+1}\right) \geq z_j.
			\end{displaymath}
\item \textbf{(G2)} for each level $j\in[m-2]$ and all 
       		 $\gamma \in (0,\gamma_0]$, if $|P_t\cap A_{\geq j}|\geq 	
        	\gamma_0\lambda$ and $|P_t\cap A_{\geq j+1}|\geq \gamma\lambda$ then
\begin{displaymath}
\Pr\nolimits_{y \sim \mathcal{D}(P_t)}\left(y \in A_{\geq j+1}\right) \geq \left(1+\delta\right)\gamma.
\end{displaymath}			

\item \textbf{(G3)} and the population size $\lambda \in \mathbb{N}$ satisfies
\begin{displaymath}
\lambda \geq \left(\frac{4}{\gamma_0\delta^2}\right)\ln\left(\frac{128m}{z_*\delta^2}\right), 
\end{displaymath}
where $z_* \coloneqq \min_{j\in [m-1]}\{z_j\}$, then
\begin{displaymath}
\mathbb{E}\left[T\right] \leq \left(\frac{8}{\delta^2}\right)\sum_{j=1}^{m-															1}\left[\lambda\ln\left(\frac{6\delta\lambda}{4+z_j\delta\lambda}\right)+\frac{1}{z_j}\right].
\end{displaymath}
\end{itemize}
\end{theorem}

Informally, the first condition (G1) requires that the 
probability of sampling an individual in level $A_{\ge j+1}$ 
 is at least $z_j$ given that at least $\gamma_0\lambda$ 
individuals in the current population are in level $A_{\ge j}$. 
Condition (G2) further 
requires that given that $\gamma_0\lambda$ individuals 
of the current population belong to levels $A_{\ge j}$, 
and, moreover, $\gamma\lambda$ of them are lying at 
levels $A_{\ge j+1}$, the probability of sampling an 
offspring in levels $A_{\ge j+1}$ is at least
$(1+\delta)\gamma$. The last condition (G3) 
sets a lower limit on the population size $\lambda$. 
As long as the three conditions are satisfied, 
an upper bound on the expected 
time to reach the last level $A_m$ 
of a population-based algorithm is guaranteed. 

To apply 
the level-based theorem, it is recommended to 
follow the five-step procedure in \cite{bib:Corus2017}: 
1) identifying a partition of 
the search space 2) finding appropriate parameter settings 
such that condition (G2) is met 3) 
estimating a lower bound $z_j$ to satisfy 
condition (G1) 4) ensuring the
the population size is large 
enough and 5) derive the upper bound on the expected 
time to reach level $A_m$.

Note in particular that 
Algorithm~\ref{abstract-algor} 
assumes a mapping 
$\mathcal{D}$ from the space of populations 
$\mathcal{X}^{\lambda}$ to the space of 
probability distributions over the search space. 
The mapping $\mathcal{D}$ is often said to depend 
on the current population only 
\cite{bib:Corus2017}; however, this is not strictly necessary.
Very recently, 
Lehre and Nguyen \cite{NguyenPbil2018} 
applied Theorem~\ref{thm:levelbasedtheorem}  
to analyse the expected optimisation time of the \pbil
with a sufficiently large offspring 
population size $\lambda=\Omega(\log n)$
on \leadingones and \binval, when the population for the next generation
in the \pbil is sampled using a mapping that depends on
the previous probabilistic model $p_t$ 
in addition to the current population $P_t$.
The rationale behind this is that, in each generation, 
the \pbil draws $\lambda$ samples from the probability 
distribution (\ref{eq:prob-distribution}), 
that correspond to $\lambda$ individuals 
in the current population. If
the number of samples $\lambda$ is 
sufficiently large, it is highly likely that 
the empirical distributions 
for all positions among the entire population 
cannot deviate too far from the true 
distributions, i.e. marginals $p_t(i)$ \cite{NguyenPbil2018},
due to the Dvoretzky--Kiefer--Wolfowitz inequality 
\cite{Massart-DKW}.

\subsection{Feige's Inequality}
In order to verify conditions (G1) and (G2) of 
Theorem~\ref{thm:levelbasedtheorem} for the 
\umda on \onemax using a canonical $f$-based 
partition $A_1,\ldots,A_m$,
we later need a lower bound on 
the probability of sampling an offspring
in given levels, that is  $\Pr_{y\sim p_t}(y \in A_{\ge j})$,
where $y$ is the offspring sampled from the 
probability distribution (\ref{eq:prob-distribution}).
Let $Y$ denote the number of ones in the offspring $y$. 
It is well-known that the random variable $Y$ 
follows a
Poisson-Binomial distribution with expectation
$\mathbb{E}\left[Y\right]=\sum_{i=1}^{n}p_t(i)$ and 
variance $\sigma_n^2=\sum_{i=1}^{n}p_t(i)\left(1-p_t(i)\right)$. 
A general result due to 
Feige \cite{bib:Feige2004} 
provides such a lower bound when $Y<\expect{Y}$; however, 
for our purposes, it will be more
convenient to use the following variant \cite{bib:Dang2015a}.
\begin{theorem}[Corollary~3 in \cite{bib:Dang2015a}]\label{cor:feige-ineq}
  Let $Y_1,\dots,Y_n$ be $n$ independent random variables with support
  in $[0,1]$, define $Y = \sum_{i=1}^{n} Y_i$
  and $\mu = \expect{Y}$. It holds for every $\Delta>0$ that
\begin{displaymath}
  \prob{Y > \mu - \Delta} \geq \min\left\{\frac{1}{13},\frac{\Delta}{1+\Delta}\right\}.
\end{displaymath}
\end{theorem}

\subsection{Anti-Concentration Bound}
In addition to Feige's inequality, 
it is also necessary to compute an upper bound on the
probability of sampling an offspring in a given level, that is 
$\Pr_{y\sim p_t}\left(y \in A_j \right)$ for any $j \in [m]$, where 
$y\sim \Pr(\cdot \mid p_t)$ as defined in (\ref{eq:prob-distribution}). 
Let $Y$ be the random variable that follows a Poisson-Binomial distribution
as introduced in the previous subsection.
Baillon et al. \cite{bib:Baillon} derived the following
sharp upper bound on the probability 
$\Pr_{y\sim p_t}\left(y \in A_j \right)$. 

\begin{theorem}
[Adapted from Theorem~2.1 in \cite{bib:Baillon}]
\label{thm:anticoncentration}
	Let $Y$ be an integer-valued random 
    variable that follows a Poisson-Binomial
    distribution with parameters $n$ and $p_t$, and let  
    $\sigma_n^2=\sum_{i=1}^{n}p_t(i)(1-p_t(i))$ be the variance of $Y$. 
    For all $n$, $y$ and $p_t$, it then holds that
	\begin{displaymath}
		\sigma_n\cdot\Pr\left(Y=y\right)\leq \eta,	
	\end{displaymath}
	where $\eta$ is an absolute constant being
    \begin{displaymath}
    \eta =\max_{x\geq 0}\sqrt{2x} 
    e^{-2x}\sum_{k=0}^{\infty}\left(\frac{x^k}{k!}\right)^2 
    \approx 0.4688.
    \end{displaymath}
\end{theorem}


\section{Runtime of the \umda on LeadingOnes and BinVal}
\label{sec:umda-leadingones}\label{sec:umda-binval}

As a warm-up example, and to illustrate the method of level-based 
analysis, we consider the 
two functions -- \leadingones and \binval -- as defined in
Definitions~\ref{def:leadingones} and~\ref{def:binval}.
It is well-known that the expected optimisation time of
the ($1$+$1$)~\ea on \leadingones is $\Theta(n^2)$, and that this is
optimal for the class of \emph{unary unbiased} black-box algorithms 
\cite{lehreblackbox2012}. Early analysis of 
the \umda  on \leadingones \cite{bib:Chen2010} 
required an excessively large 
population, \ie $\lambda = \omega(n^2\log n)$. 
Our analysis below 
 shows that  a population size 
$\lambda = \Omega(\log n)$  suffices to achieve
the expected optimisation time $\mathcal{O}(n^2)$.

\binval is a linear function with exponentially decreasing weights
relative to the bit position.
Thus, the function is often regarded as an extreme linear 
function (the other one is \onemax) \cite{bib:Droste2006}. 
Droste \cite{bib:Droste2006} was the first to prove
an upper bound of $\bigO{nK}=\bigO{n^{2+\varepsilon}}$
on the expected optimisation time of the \cga on \binval, assuming 
that $\varepsilon>0$ is a constant. Regardless of the abstract 
population size $K$, 
Witt recently
derived a  lower bound of $\Omega(n^2)$
on the expected optimisation time of the \cga on \binval 
\cite[Corollary~3.5]{Witt2018Domino} and
verified the claim made earlier by 
Droste \cite{bib:Droste2006} that \binval harder problem than \onemax
for the \cga.
We now give our runtime bounds for the \umda on 
\leadingones and \binval with a sufficiently large 
population size $\lambda$.

\begin{theorem}\label{thm:leadingones}
	The \umda (with margins) with parent population size $\mu \geq c \log{n}$ for a 
	sufficiently large constant $c>0$, and offspring population size 
	$\lambda \geq (1+\delta)e\mu$ for any constant $\delta>0$, has expected optimisation time 
	$\mathcal{O}(n\lambda\log{\lambda}+n^2)$ on \leadingones and \binval.
\end{theorem}
\begin{proof}
	We apply Theorem \ref{thm:levelbasedtheorem} 
    by following the guidelines from \cite{bib:Corus2017}.
	
	\textbf{Step 1}: For both functions, we define the levels
	\begin{displaymath}
			A_j := \{ x\in\{0,1\}^n \mid \leadingones(x)=j-1\}.
	\end{displaymath}
    
	Thus, there are $m = n+1$ levels ranging from $A_1$ to $A_{n+1}$. 
    Note that a constant 
    $\gamma_0$ appearing later in this proof is set to
    $\gamma_0 := \mu/\lambda$, that coincides with the selective pressure
    of the \umda.
	
	For \leadingones, the partition is clearly $f$-based as it is canonical 
	to the function. 
	For \binval, however, note that since all the
	$j-1$ leading bits of any $x \in A_j$ are ones, then the contribution of these
	bits to $\binval(x)$ is $\sum_{i=1}^{j-1}2^{n-i}$. On the other hand, the contribution 
	of bit position $j$ is $0$, and that of the last $n-j$ bits is between $0$ 
	and $\sum_{i=j+1}^n 2^{n-i} = \sum_{i=0}^{n-j-1} 2^i = 2^{n-j}-1$, 
	so in overall 
	\begin{displaymath}
		\sum_{i=1}^{j}2^{n-i} - 1 
		\geq \binval(x) 
		\geq \sum_{i=1}^{j-1}2^{n-i}.
	\end{displaymath}

Therefore, for any $j \in [n]$ and all $x \in A_j$, and all $y \in A_{j+1}$ 
	we have that 
\begin{displaymath}
\binval(y) 
		\geq \sum_{i=1}^{j}2^{n-i} 
		> \sum_{i=1}^{j}2^{n-i}-1 
		\geq \binval(x);
\end{displaymath}
thus, the partition is also $f$-based for \binval.
	This observation allows us to carry 
      over the proof arguments of \leadingones to 
	  \binval.
	
	\textbf{Step 2}: In (G2), for any level 
    $j \in [n-1]$ satisfying $|P_{t} \cap 
	A_{\geq j}| \geq \gamma_0\lambda = \mu$ and 
    $|P_{t} \cap A_{\geq j+1}| \geq \lambda 
	\gamma$ for some $\gamma \in (0,\gamma_0]$, 
    we seek a lower bound $(1+\delta)\gamma$ 
	for $\prob{y \in A_{\geq j+1}}$ where 
    $y \sim \mathcal{D}(P_t)$. 
    The given conditions 
	on $j$ imply that the $\mu$ fittest individuals of
    $P_t$ have at least $j-1$ leading $1$-bits 
	and among them at least $\lceil \gamma\lambda \rceil$ have at least $j$ leading 
	$1$-bits. Hence, $p_{t+1}(i)=1-1/n$ for $i \in [j-1]$ and $p_{t+1}(j) \geq \max(\min(1-1/n,
	\gamma\lambda/\mu),1/n)\geq \min(1-1/n,\gamma/\gamma_0)$, so
	\begin{align*}
		\prob{y \in A_{\geq j+1}}
		&\geq \prod_{i=1}^{j} p_{t+1}(i)\\
		&\geq \min\left\{
		\left(1 - \frac{1}{n}\right)^{j},
		\left(1 - \frac{1}{n}\right)^{j-1} \cdot \frac{\gamma\lambda}{\mu}\right\}\\
		&\geq \min\bigg\{\frac{1}{e}, \frac{\gamma}{e\gamma_0}\bigg\} \\
        &= \frac{\gamma}{e\gamma_0}
        = \frac{\lambda\gamma}{e\mu} \geq (1 +\delta)\gamma,
	\end{align*}
	due to $\gamma\leq\gamma_0$ and $\lambda\ge (1+\delta)e\mu$
    for any constant $\delta>0$. Therefore, condition (G2)  is now satisfied.
	
	\textbf{Step 3}: In (G1), 
    for any level $j \in [n]$ satisfying $|P_{t} \cap A_{\geq j}
	| \geq \gamma_0\lambda = \mu$ we need a lower bound $\prob{y \in 
		A_{\geq j+1}}\ge z_j$. 
        Again the condition on level $j$ gives that the $\mu$ fittest individuals 
	of $P_t$ have at least $j-1$ leading $1$-bits, or $p_{t+1}(i)=1-\frac{1}{n}$ for $i \in [j-1]$. 
	Due to the imposed lower margin, we can assume pessimistically that $p_{t+1}(j) = \frac{1}{n}$. 
	Hence,
	\begin{align*}
		\prob{y \in A_{\geq j+1}}
		&\geq \prod_{i=1}^{j} p_{t+1}(i)\\
		&\geq \left(1 - \frac{1}{n}\right)^{j-1} \cdot\frac{1}{n} 
		= \frac{1}{en} =: z_j.
	\end{align*}
	So, (G1) is satisfied for $z_j:=\frac{1}{en}$.
	
	\textbf{Step 4}: Considering (G3), because $\delta$ is a constant, and
	both $1/z_*$ and $m$ are $\mathcal{O}(n)$, 
    there must exist a constant $c>0$ such that 
	$\mu \geq c \log n \geq (4/\delta^2)\ln(128 m / (z_* \delta^2))$. Note that 
	$\lambda = \mu/\gamma_0$, so (G3) is satisfied.
	
	\textbf{Step 5}: 
    All conditions of Theorem~\ref{thm:levelbasedtheorem} are 
	satisfied, so the expected optimisation time of the \umda on \leadingones is
	\begin{align*}
		\expect{T} 
		&= \mathcal{O}\left(\sum_{j=1}^{n} \left( \lambda \ln\left(\frac{\lambda}{1 + \lambda/n}\right) + n \right)\right)\\
		&= \mathcal{O}\left(\lambda\log\lambda + n^2\right).
	\end{align*}
    
       We now consider \binval. In both problems, all that matters to determine 
the level of a  bitstring is the position of the first zero-bit.
Now consider two bitstrings in the same level for \binval, 
their rankings after the population is sorted are also determined 
by some other less significant bits; however, the proof thus far
never takes these bits into account.
Hence, the expected optimisation time of the \umda 
on \leadingones can be carried over to \binval 
for the \umda with margins using truncation selection.    
\end{proof}


\section{Runtime of the UMDA on OneMax}\label{sec:umda-onemax}

We consider the problem in Definition~\ref{def:onemax}, i.e.,
maximisation of  the number of ones in a 
bitstring.
It is well-known that \onemax can be optimised in expected
time
$\Theta(n\log n)$ using the simple $(1+1)$~\ea. 
The level-based theorem yielded the first upper bound 
on the expected optimisation time of the \umda on \onemax, which is 
$\mathcal{O}(n\lambda\log \lambda)$, assuming that
$\lambda=\Omega(\log n)$ \cite{bib:Dang2015a}. This leaves open 
whether the \umda is slower than the $(1+1)~\ea$ 
and other traditional \eas on \onemax.

We now introduce additional notation used throughout the section.
The following random variables related to the sampling of a Poisson Binomial 
distribution with the parameter vector 
$p_t = \left(p_t(1),\dots,p_t(n)\right)$ are
often used in the proofs.
\begin{itemize}
	\item Let $Y:=(Y_1,Y_2,\ldots,Y_n)$ denote
	an offspring sampled from the probability distribution
    (\ref{eq:prob-distribution}) in generation $t$, where
	$\Pr(Y_i=1)=p_t(i)$ for each $i\in [n]$.
	\item Let $Y_{i,j}:=\sum_{k=i}^j Y_k$ denote the number 
    of ones sampled from the  sub-vector 
    $\left(p_t(i),p_t(i+1),\ldots,p_t(j)\right)$ of the model $p_t$ where 
    $1\leq i\leq j\leq n$.
\end{itemize}

\subsection{Small parent population size}

Our approach refines the
analysis in \cite{bib:Dang2015a} by considering 
anti-concentration properties of the random variables involved.  As
already discussed in subsection~\ref{level-based-theorem-section}, we
need to verify the 
three conditions (G1), (G2) and (G3) of Theorem~\ref{thm:levelbasedtheorem} 
to derive an
upper bound on the expected optimisation time. 
The range of values of the  marginals are
\begin{displaymath}
	p_t(i) \in \left\{\frac{k}{\mu}\mid k\in[\mu-1]\right\}\cup
        \left\{1-\frac{1}{n}, \frac{1}{n} \right\}.
\end{displaymath}
When $p_t(i)=1-1/n$ or $1/n$, we say that 
the marginal is at the upper
or lower border (or margin), respectively. Therefore, 
we can categorise values for $p_t(i)$ into three groups: those at the upper
  margin $1-1/n$, those at the lower margin $1/n$, and those within the
  closed interval $[1/\mu, 1-1/\mu]$. 
  For \onemax, all bits have the same weight and the fitness 
  is just the sum of these bit values, so the re-arrangement 
  of bit positions will  have no impact on the 
  sampling distribution. 
  Given the current sorted population, 
  recall that $X_i:=\sum_{k=1}^{\mu} x_{t,i}^{(k)}$, and 
  without loss of generality, we
  can re-arrange the bit-positions so that for two integers
  $k,\ell\geq 0$, it holds
  \begin{itemize}
  \item for all $i\in[1,k],$ $1\leq X_i \leq \mu-1$ and $p_t(i)=X_i/\mu$, 
  \item for all $i\in(k,k+\ell]$, $X_i = \mu$ and $p_t(i)=1-1/n$, and
  \item for all $i\in(k+\ell,n]$, $X_i =0$ and $p_t(i)=1/n$.
  \end{itemize}	
We define the levels using the  
canonical $f$-based partition
\begin{align}
A_j & \coloneqq \left\{x\in \{0,1\}^n \mid \onemax(x)=j-1\right\}.
\label{eq:level-def}
\end{align}
 Note that the probability appearing in conditions (G1) and (G2) of
  Theorem~\ref{thm:levelbasedtheorem} 
  is the probability of sampling 
  an offspring in levels $A_{\ge j+1}$, 
  that is $\prob{Y_{1,n}\geq j}$.

We aim at obtaining an upper bound of $\mathcal{O}(n\lambda)$ 
on the expected optimisation time 
of the \umda on \onemax using the level-based theorem. 
The logarithmic factor $\mathcal{O}(\log\lambda)$ in the previous upper 
bound $\mathcal{O}(n\lambda\log \lambda)$ in 
\cite{bib:Dang2015a} stems from the lower bound $\Omega(1/\mu)$
on the parameter $z_j$ in the condition (G1) of 
Theorem~\ref{thm:levelbasedtheorem}.
We aim for the
stronger bound $z_j=\Omega(\frac{n-j+1}{n})$. 
Note that in the following proofs, we choose the parameter
$\gamma_0:=\mu/\lambda$.

Assume that the current level is $A_j$, that is 
$|P_t\cap A_{\ge j}|\ge \gamma_0\lambda=\mu$, 
which, together with the two variables $k$ and $\ell$, 
implies that there are at least $j-\ell-1$ 
ones from the first $k$ bit positions. 
To verify conditions (G1) and (G2) of Theorem~\ref{thm:levelbasedtheorem}, we 
need to calculate the probability of sampling an offspring with at least 
$j$ ones (in levels $A_{\ge j+1}$). It is thus more likely for
the algorithm to maintain the 
$\ell$ ones for all bit positions 
$i\in (k,k+\ell]$ (actually this happens with
probability at least $1/e$),
and also sample at least $j-\ell$ 
ones from the remaining $n-\ell$ bit positions.
This lead
us to consider three distinct cases according to different 
configurations of the current population with respect to 
the two parameters $k$ and $j$ in Step 3 of 
Theorem~\ref{onemax-small-mu} below.  
\begin{enumerate}
\item $k\geq \mu$. In this situation, the variance of 
$Y_{1,k}$ is not too small. By the result of 
Theorem~\ref{thm:anticoncentration},
the distribution of $Y_{1,k}$ cannot be
too concentrated on its mean 
$\mathbb{E}\left[Y_{1,k}\right]=j-\ell-1$, and  
with probability at least $\Omega(1)$, 
the algorithm can sample at least $j-\ell$ ones
from the first $k$ bit positions 
to obtain an offspring with at least 
$(j-\ell)+\ell=j$ ones. Thus, 
the probability of sampling at least 
$j$ ones is bounded from below by 
\begin{align*}
\Pr(Y_{1,n}\geq j)
&\geq \Pr(Y_{1,k}\geq j-\ell) \Pr(Y_{k+1, k+\ell}=\ell)\\
&=\Omega(1).
\end{align*}

\item $k<\mu$ and $j\geq n+1-\frac{n}{\mu}$. In this case, 
the current level is very close to the optimal $A_{n+1}$, 
and the bitstring has few zeros. 
As already obtained from \cite{bib:Dang2015a}, 
the probability of sampling an offspring in 
$A_{\ge j+1}$ in this case is $\Omega(\frac{1}{\mu})$. 
Since the condition can be rewritten as 
$\frac{1}{\mu}\geq \frac{n-j+1}{n}$, it ensures that 
$z_j=\Omega(\frac{1}{\mu})=\Omega(\frac{n-j+1}{n})$.  

\item The remaining cases. Later will we prove that if
  $\mu\leq \sqrt{n(1-c)}$ for some constant $c\in (0,1)$, and excluding the 
  two cases above, imply
  $0\leq k<(1-c)(n-j+1)$.  In this case, $k$ is relatively small, and
  $\ell$ is not too large since the current level is not very close to
  the optimal $A_{n+1}$. This implies that most 
  zeros must be located among bit positions 
  $i\in (k+\ell,n]$, and it suffices to sample an extra one from
  this region to get at least $(j-\ell-1)+\ell+1=j$ ones. 
  The probability of sampling an
  offspring in levels $A_{\geq j+1}$ is then 
  $z_j=\Omega(\frac{n-j+1}{n})$.
\end{enumerate} 

We now present our detailed runtime analysis 
for  the \umda on \onemax, when the 
population size is small, that is,  
$\mu = \Omega(\log n) \cap \mathcal{O}(\sqrt{n})$.
     
\begin{theorem}\label{onemax-small-mu}
  For some constant $a>0$ and any constant $c\in(0,1)$,
  the \umda (with margins) with parent population size
  $a\ln(n)\leq \mu\leq \sqrt{n(1-c)}$, and
  offspring population size $\lambda\geq (13e/(1-c))\mu$, has
  expected optimisation time $\mathcal{O}\left(n\lambda\right)$ on \onemax.
\end{theorem}
\begin{proof}
  Recall that $\gamma_0\coloneqq \mu/\lambda$. We re-arrange the 
  bit positions as explained above and 
  follow the recommended 5-step 
procedure for applying Theorem~\ref{thm:levelbasedtheorem}
  \cite{bib:Corus2017}. 
	
  \textbf{Step 1.} The levels are defined as in Eq.~(\ref{eq:level-def}).
  There are exactly $m=n+1$ levels from $A_1$ to $A_{n+1}$, 
  where level $A_{n+1}$ consists of the optimal solution. 
	
  \textbf{Step 2.} We verify condition (G2) of Theorem~\ref{thm:levelbasedtheorem}. In particular, for some $\delta\in(0,1)$, for any level 
  $j\in [m-2]$ and any  $\gamma\in (0,\gamma_0]$, assuming 
  that the population is configured  such that
  $|P_t\cap A_{\geq j}|\geq \gamma_0\lambda=\mu$ and
  $|P_t\cap A_{\geq j+1}|\geq \gamma\lambda>0$, we must show that the
  probability of sampling an offspring 
  in levels $A_{\ge j+1}$ must be no
  less than $(1+\delta)\gamma$. By
  the re-arrangement of the bit-positions mentioned earlier, it holds that
  \begin{align}
    \sum_{i=k+1}^{k+\ell}X_i=\mu \ell \quad 
    \mbox{ and }\quad \sum_{i=k+\ell+1}^{n}X_i=0,	\label{eq:1}
  \end{align}	
  where $X_i$ for all  $ i\in[n]$ are given in Algorithm \ref{umda-algor}.
  By assumption, the current population $P_t$ consists of 
  $\gamma\lambda$ individuals with at least $j$ ones 
  and $\mu-\gamma\lambda$ individuals with exactly $j-1$ ones. Therefore, 
  \begin{align}\label{eq:2}
    \sum_{i=1}^{n}X_i\geq \gamma\lambda j + \left(\mu-\gamma\lambda\right)(j-1) 
    = \gamma\lambda+\mu\left(j-1\right).
  \end{align}
  Combining (\ref{eq:1}), (\ref{eq:2}) and noting that $\lambda=\mu/\gamma_0$ yield
  \begin{align*}		
    \sum_{i=1}^{k}X_i 
      &=\sum_{i=1}^{n}X_i-\sum_{i=k+1}^{k+\ell}X_i-\sum_{i=k+\ell+1}^{n}X_i \\
     &\geq \gamma\lambda+\mu\left(j-1\right)-\mu\ell=\mu\left(\frac{\gamma}{\gamma_0}+j-1-\ell\right).	
  \end{align*}
  Let $Z=Y_{1,k}+Y_{k+\ell+1,n}$ be the
  integer-valued random variable, which describes the 
  number of ones sampled
  in the first $k$ and the last $n-k-\ell$ bit positions. Since
  $k+\ell\leq n$, the expected value of $Z$ is
  \begin{align}  \label{eq:expectation-of-z} 
  \begin{split}
      \mathbb{E}\left[Z\right]
      &= \sum_{i=1}^{k}p_t(i) +\sum_{i=k+\ell+1}^{n}p_t(i) \\
      &= \frac{1}{\mu}\sum_{i=1}^{k}X_i +\frac{n-k-\ell}{n}
      \geq j-\ell-1+\frac{\gamma}{\gamma_0}.
  \end{split}
  \end{align}	
  In order to obtain an offspring with at least $j$ ones, it is
  sufficient to sample $\ell$ ones in positions
  $k+1$ to $k+\ell$ and at least $j-\ell$ ones from
  the other positions. The probability of this event is bounded from below by
  \begin{align}   
      \Pr\left(Y_{1,n} \geq j\right) \geq \Pr\left(Z\geq j-\ell\right)\cdot \Pr\left(Y_{k+1,k+\ell}=\ell\right). \label{eq:g2-upgrade-prob}  
  \end{align}  	        
 The probability to obtain $\ell\ge n-1$ 
 ones in the middle interval from
 position $k+1$ to $k+\ell$ is 
 \begin{align}
   \Pr(Y_{k+1,k+\ell}=\ell)=\left(1-\frac{1}{n}\right)^\ell\geq \left(1-\frac{1}{n}\right)^{n-1}\geq \frac{1}{e}\label{eq:no-mut-ell-prob}
 \end{align}	
 by the result of Lemma~\ref{exponential-bounds} for $t=-1$.
  We now estimate the probability $\Pr\left(Z\geq j-\ell\right)$
  using Feige's inequality. 
  Since $Z$ takes integer values only, it follows by 
  (\ref{eq:expectation-of-z}) that
  \begin{align*}	
      \Pr\left(Z\geq j-\ell\right)&=\Pr\left(Z> j-\ell-1\right)\\
      &\geq\Pr\left(Z>\mathbb{E}\left[Z\right]-\frac{\gamma}{\gamma_0}\right).	
  \end{align*}		
  Applying Theorem~\ref{cor:feige-ineq} for
  $\Delta =\gamma/\gamma_0\leq 1$ and noting that we chose $\mu$ and
  $\lambda$ such that $1/\gamma_0=\lambda/\mu\geq 13e/(1-c)= 13e(1+\delta)$ yield
  \begin{align}
  \begin{split}
  \Pr\left(Z\geq j-\ell\right)
     &\geq \min\bigg\{\frac{1}{13},\frac{\Delta}{\Delta+1}\bigg\}\\
     &\geq \frac{\Delta}{13}  = \frac{\gamma}{13\gamma_0} \geq e\left(1+\delta\right)\gamma.\label{eq:g2-rest-upgrade}
  \end{split}
  \end{align}	
  Combining (\ref{eq:g2-upgrade-prob}),
  (\ref{eq:no-mut-ell-prob}), and (\ref{eq:g2-rest-upgrade}) yields
  \begin{math}
    \Pr(Y_{1,n}\geq j)\geq \left(1+\delta\right)\gamma	
  \end{math},
  and, thus, condition (G2) of Theorem~\ref{thm:levelbasedtheorem} holds.
	
  \textbf{Step 3.} We now consider condition (G1) for any 
  level $j$. Let $P_t$ be any population where $|P_t\cap A_{\geq j}|\geq
  \gamma_0\lambda=\mu$. For a lower bound on $\prob{Y_{1,n}\geq j}$, we
  modify the population such that any individual in levels $A_{\geq
    j+1}$ is moved to level $A_j$. Thus, the $\mu$ fittest
  individuals belong to level $A_j$. By the definition of the \umda,  
  this will only reduce the probabilities $p_{t+1}(i)$ on the \onemax
  problem. Hence, by Lemma~\ref{lem:stocdom}, the distribution of $Y_{1,n}$
  for the modified population is stochastically dominated by $Y_{1,n}$ for
  the original population. A lower bound $z_j$ that holds for the modified
  population therefore also holds for the original population.
  All  the $\mu$ fittest
  individuals in the current sorted population $P_t$ 
  have exactly $j-1$ ones, and, therefore,
  \begin{math}
    \sum_{i=1}^{n}X_i= \mu\left(j-1\right)	
  \end{math}
  and
  \begin{math}
    \sum_{i=1}^{k}X_i= \mu\left(j-\ell-1\right)	
  \end{math}.	
  There are four distinct cases that cover all 
  situations according to different values of 
  variables $k$ and $j$. We aim to show that in all
  four cases, we can use the parameter $z_j=\Omega(\frac{n-j+1}{n})$.

  \underline{Case 0:}  $k=0$. In this case, $p_t(i)=1-1/n$
  for $1\leq i\leq j-1$, and $p_t(i)=1/n$ for $j\leq i\leq n$. To
  obtain $j$ ones, it suffices to sample only ones in
  the first $j-1$ positions, and exactly a one in the remaining
  positions, i.e.,
  \begin{align*}
    \prob{Y_{1,n}\geq j} &\geq \frac{n-j+1}{n}\left(1-\frac{1}{n}\right)^{n-1}\\
                        &=\Omega\left(\frac{n-j+1}{n}\right).
  \end{align*}

  \underline{Case 1:}  $k\geq \mu$. We will apply the anti-concentration inequality
  in Theorem~\ref{thm:anticoncentration}. To lower bound the
  variance of the number of ones sampled in the first $k$ positions,
  we use the bounds $1/\mu\leq p_i(t)\leq 1-1/\mu$ which hold for
  $1\leq i\leq k$. In particular,
  \begin{align}
  \begin{split}
   \var{Y_{1,k}}&=\sum_{i=1}^{k}p_t(i)\left(1-p_t(i)\right)\\
    &\geq 
    		\frac{k}{\mu}\left(1-\frac{1}{\mu}\right)\geq \frac{9k}{10\mu}\geq \frac{9}{10},
  \end{split}
  \end{align}
  where the second inequality holds for sufficiently large $n$ 
  because $\mu\geq a\ln(n)$ for some constant $a>0$. 
 Theorem~\ref{thm:anticoncentration} applied with $\sigma_k\geq
 \sqrt{9/10}$ now gives
 \begin{align*}
   \prob{Y_{1,k}=j-\ell-1}\leq \eta/\sigma_k.
 \end{align*}
  Furthermore, since  $\expect{Y_{1,k}}$ is an integer,
  Lemma~\ref{int-expectation} implies that
  \begin{align}\label{eq:sample-beyond-mean}
    \prob{Y_{1,k}\geq \expect{Y_{1,k}}} \geq 1/2.
  \end{align}
  By combining these two probability bounds, 
  the probability of sampling an offspring with
  at least $j-\ell$ ones 
  from the first $k$ positions is 
 \begin{align*}
   &\Pr\left(Y_{1,k}\geq j-\ell\right)\\
   &=\Pr\left(Y_{1,k}\geq j-\ell-1\right)-\Pr\left(Y_{1,k}=j-\ell-1\right)\nonumber\\
   &=\Pr\left(Y_{1,k}\geq \mathbb{E}\left[Y_{1,k}\right]\right)-\Pr\left(Y_{1,k}=j-\ell-1\right)\nonumber\\
     &\geq \frac{1}{2}-\frac{\eta}{\sigma_k} >\frac{1}{2}-\frac{0.4688}{\sqrt{9/10}} = \Omega(1).\label{eq:case1-pos-upgrade}
 \end{align*}
	In order to obtain an offspring in levels $A_{\geq j+1}$,
    it is sufficient to sample at least $j-\ell$ ones
        from the $k$ first positions and $\ell$ ones from position
        $k+1$ to position $k+\ell$. Therefore, 
        using (\ref{eq:no-mut-ell-prob}) and the above lower bound,        
        this event happens with probability bounded from below by
        
	\begin{align*}
	\Pr\left(Y_{1,n}\geq j\right)&\geq \Pr\left(Y_{1,k}
    \geq j-\ell\right)\cdot\Pr\left(Y_{k+1,k+\ell}=\ell\right)\\
	&> \Omega(1)\cdot\frac{1}{e} =\Omega\left(\frac{n-j+1}{n}\right).
	\end{align*} 
	
 	\underline{Case 2:} $1\leq k<\mu$ and $j\geq n(1-1/\mu)+1$. The 
    second condition is equivalent to $1/\mu\geq (n-j+1)/n$. The probability
        of sampling an offspring in levels $A_{\geq j+1}$ is
        then bounded
        from below by
	\begin{align*}
	&\Pr\left(Y_{1,n}\geq j\right)\\
    &\geq
        \Pr\left(Y_{1,1}=1\right)
        \Pr\left(Y_{2,k}\geq j-\ell-1\right)
        \Pr\left(Y_{k+1,k+\ell}=\ell\right)\\
        &\geq \frac{1}{\mu}\Pr\left(Y_{2,k}\geq j-\ell-1\right)\frac{1}{e}
         \geq \frac{1}{14e\mu},
       \end{align*}
       where we used the inequality $\Pr\left(Y_{2,k}\geq
         j-\ell-1\right)\geq 1/14$ for $\mu\geq 14$ 
         proven in \cite{bib:Dang2015a}. Since 
        $1/\mu\geq (n-j+1)/n$, we can conclude that
	\begin{displaymath}
	\Pr\left(Y_{1,n}\geq j\right)
           \geq \frac{1}{14e\mu}
           \geq \frac{n-j+1}{14en}=\Omega\left(\frac{n-j+1}{n}\right).
	\end{displaymath}	

	\underline{Case 3:} $1\leq k<\mu$ and $j<n(1-1/\mu)+1$. This
case covers all the remaining situations not included by the
first two cases. The latter inequality can be 
rewritten as $n-j+1\geq n/\mu$. We also have 
$\mu\leq \sqrt{n(1-c)}$, so $n/\mu\geq \mu/(1-c)$. It then holds that 
\begin{displaymath}
	(1-c)(n-j+1)  \geq (1-c)(n/\mu) \geq(1-c)\mu/(1-c)=\mu>k.
\end{displaymath}
Thus, the two conditions can be shortened to 
$1\leq k<(1-c)(n-j+1)$. In this case, 
the probability of sampling $j$ ones is 
\begin{align*}
	 &\Pr(Y_{1,n}\geq j)\\	
     &\geq \Pr\left(Y_{1,k}\geq j-\ell-1\right)
	\Pr\left(Y_{k+1,n}\ge \ell+1\right)\\
	&\geq \frac{1}{2}\cdot\frac{1}{e}\cdot\frac{n-k-\ell}{n} 
	= \frac{n-k-\ell}{2en},
\end{align*}
where the $1/2$ factor in the last inequality is due to \eqref{eq:sample-beyond-mean}.
Since $\ell\leq j-1$ and $k<(1-c)(n-j+1)$, it follows that 
\begin{align*}
	\Pr\left(Y_{1,n}\geq j\right)
    &> \frac{n-(1-c)(n-j+1)-j+1}{2en} \\
    &=\Omega\left(\frac{n-j+1}{n}\right).
\end{align*}	
Combining all three cases together yields the probability of sampling an offspring in levels $A_{\ge j+1}$ as follows.
\begin{displaymath}
\Pr\left(Y_{1,n}\geq j\right)  = \Omega\left(\frac{n-j+1}{n}\right),
\end{displaymath}
and by defining 
$z_j=c\cdot \frac{n-j+1}{n}$ for a sufficiently small $c>0$ and 
choosing $z_*\coloneqq \min_{j\in[n]}\{z_j\}=\Omega(1/n)$, condition (G1) of 
Theorem~\ref{thm:levelbasedtheorem} is satisfied.

	\textbf{Step 4.} We consider condition (G3) regarding the 
    population size. We have $1/\delta^2=\mathcal{O}(1)$,
    $1/z_*= \mathcal{O}(n)$, and $m=\mathcal{O}(n)$. Therefore, there must 
    exist a constant $a>0$ such that	    
    \begin{align*}
      \left(\frac{a}{\gamma_0}\right)\ln(n) \geq \left(\frac{4}{\gamma_0\delta^2}\right)\ln\left(\frac{128m}{z_*\delta^2}\right). 
    \end{align*}
    The requirement $\mu\geq a\ln(n)$ now implies that
    \begin{align*}
      \lambda 
      =  \frac{\mu}{\mu/\lambda} 
      \geq \left(\frac{a}{\gamma_0}\right)\ln(n)
      \geq \left(\frac{4}{\gamma_0\delta^2}\right)\ln\left(\frac{128m}{z_*\delta^2}\right);
    \end{align*}
    hence, condition (G3) is satisfied. 

	\textbf{Step 5.} We have verified all three
    conditions (G1), (G2), and (G3). By
        Theorem~\ref{thm:levelbasedtheorem} and the bound $z_j=\Omega((n-j+1)/n)$,
        the expected optimisation time is therefore
	\begin{align*}
          \mathbb{E}\left[T\right]=
        \mathcal{O}\left(\lambda\sum_{j=1}^{n}\ln\left(\frac{n}{n-j+1}\right)+
    	\sum_{j=1}^{n}\frac{n}{n-j+1}\right).
	\end{align*}		
	We simplify the two terms 
    separately. By 
    Stirling's approximation (see Lemma~\ref{stirling}), 
    the first term is
	\begin{multline*}
          \mathcal{O}\left(\lambda\sum_{j=1}^{n}\ln\left(\frac{n}{n-j+1}\right)\right)
    =\mathcal{O}\left(\lambda\ln\prod_{j=1}^n\frac{n}{n-j+1}\right)\\
    =\mathcal{O}\left(\lambda\ln \left(\frac{n^n}{n!}\right)\right)
    =\mathcal{O}\left(\lambda\ln \frac{n^n\cdot e^n}{n^{n+1/2}}\right)
   = \mathcal{O}\left(n\lambda\right).
	\end{multline*}
	The second term is
	\begin{displaymath}
\mathcal{O}\left(\sum_{j=1}^{n}\frac{n}{n-j+1}\right)
=\mathcal{O}\left(n\sum_{k=1}^{n}\frac{1}{k}\right) 
	= \mathcal{O}\left(n\log n\right).	
	\end{displaymath}
	Since $\lambda > \mu= \Omega(\log n)$,
    the expected optimisation time is 
    \begin{displaymath}
    \mathbb{E}\left[T\right]=\mathcal{O}\left(n\lambda\right)+
    	\mathcal{O}\left(n\log n\right)=\mathcal{O}\left(n\lambda\right).\qedhere
    \end{displaymath}
\end{proof}


\subsection{Large parent population size}
For larger parent population sizes, i.e., 
$\mu= \Omega(\sqrt{n}\log n)$,
we prove the upper bound of $\mathcal{O}(\lambda\sqrt{n})$
on the expected optimisation time of the \umda on \onemax. 
Note also that Witt \cite{bib:Witt2017}
obtained a similar result, and we rely on one of his lemmas to derive
our result. In overall, our proof is not only significantly simpler
but also
holds for different settings of $\mu$
and $\lambda$, that is, $\lambda=\Omega(\mu)$ instead of 
$\lambda=\Theta(\mu)$.

\begin{theorem}\label{onemax-large-mu}
  For sufficiently large constants $a>1$ and $c>0$,
  the \umda (with margins) with offspring population size 
  $\lambda \geq a\mu$, and parent population size
  $\mu \geq c\sqrt{n}\log n$, has
  expected optimisation time $\mathcal{O}\left(\lambda\sqrt{n}\right)$ 
  on \onemax.
\end{theorem}

Here, we are mainly interested in the parent population 
size $\mu \geq c\sqrt{n}\log n$ for a sufficiently large constant $c>0$. 
In this case, Witt~\cite{bib:Witt2017} found that 
$\Pr(T\leq n^{cc'})=\mathcal{O}(n^{-cc'})$,
where $c'$ is another positive constant and
$T:=\min\{t\geq 0\;|\;p_t(i)\leq 1/4\}$ 
for an arbitrary bit $i\in [n]$. This result implies that 
the probability of not sampling at least 
an optimal solution within $n^{cc'}$ generations is bounded by 
$\mathcal{O}(n^{-cc'})$. 
Therefore, the \umda needs 
$\mathcal{O}(n\lambda \log \lambda)/\lambda = \mathcal{O}(n\log \lambda)$ 
generations \cite{bib:Dang2015a} with probability $\mathcal{O}(n^{-cc'})$ and 
$\mathcal{O}(\lambda\sqrt{n})/\lambda=\mathcal{O}(\sqrt{n})$
with 
probability $1-\mathcal{O}(n^{-cc'})$
to optimise \onemax. 
The expected number of generations is
\begin{displaymath}
\mathcal{O}(n^{-cc'})\cdot \mathcal{O}(n\log \lambda) + (1-\mathcal{O}(n^{-cc'}))\cdot \mathcal{O}(\sqrt{n})
\end{displaymath}
If we choose the constant $c$ large enough, then $n\log \lambda$ can
subsume any polynomial number of generations, \ie 
$n\log \lambda \in \text{poly}(n)$, 
which leads to $\mathcal{O}(n^{-cc'})\cdot \mathcal{O}(n\log \lambda) = \mathcal{O}(1)$. Therefore,
the overall expected number of generations 
is still bounded by $\mathcal{O}(\sqrt{n})$, so the expected
optimisation time is $\mathcal{O}(\lambda\sqrt{n})$.

In addition, the analysis by Witt \cite{bib:Witt2017} implies
that all marginals will generally 
move to higher values and are unlikely to drop
by a large distance. We then pessimistically assume 
that all marginals are lower bounded by a 
constant $p_{\min}=1/4$. Again, we rearrange the bit positions such that
there exist two integers $0\leq k,\ell\leq n$, where $k+\ell=n$ and 
\begin{itemize}
\item $p_t(i)\in \left[p_{\min}, 1-\frac{1}{\mu}\right]$ for all $1\leq i \leq k$,
\item $p_t(i) =1-\frac{1}{n}$ for all $k+1\leq i \leq n$.
\end{itemize}
Note that $k>0$ because if $k=0$ we would have sampled
a globally optimal solution.

\begin{proof}[Proof of Theorem~\ref{onemax-large-mu}]
We  apply Theorem~\ref{thm:levelbasedtheorem} (i.e. level-based analysis).

\textbf{Step 1}:
We partition the search space into the $m$ subsets 
$A_1, \ldots,A_m$ (i.e. levels) defined for $i \in [m-1]$ as follows
\begin{align*}
  A_{i} 
    &:= \{ x \in \{0,1\}^n \mid f_{i-1} \leq \onemax(x) < f_{i} \}, \\
  \text{ and } 
  A_{m} 
    &:= \{ 1^n \},
\end{align*}
where the sequence $(f_i)_{i \in \mathbb{N}}$ 
is defined with some constant $d\in (0,1]$ as
\begin{align}
  f_0 
    &:= 0
  \text{ and }
  f_{i+1}
     := f_{i} + \lceil d\sqrt{n-f_{i}} \rceil. \label{eqn:recur-levels}
\end{align}
The range of $d$ will be specified later, but for now 
note that $m = \min\{i \mid f_{i} = n\} + 1$ and due 
to Lemma~\ref{lem:no-of-levels}\footnote{This 
and some other lemmas are in the Appendix},
we know that the sequence $(f_i)_{i \in \mathbb{N}}$ is well-behaved: it 
starts at $0$ and  increases steadily
(at least $1$ per level), then eventually
reaches $n$ exactly and remains there
afterwards. Moreover, the number 
of levels satisfies $m = \Theta(\sqrt{n})$.

\textbf{Step 2}:
For (G2), we assume that 
$|P_t\cap A_{\geq j}|\geq \gamma_0\lambda=\mu$ and 
$|P_t\cap A_{\geq j+1}|\geq \gamma\lambda$. 
Additionally, we make the pessimistic assumption that 
$|P_t\cap A_{\geq j+2}|= 0$, i.e. 
the current population contains exactly $\gamma\lambda$ individuals in
$A_{j+1}$, $\mu-\gamma\lambda$ individuals in level $A_{j}$,
and $\lambda-\mu$ individuals in the levels below $A_j$.
In this case,
\begin{align*}
\sum_{i=1}^n X_i &= \gamma\lambda f_j +(\mu-\gamma\lambda)f_{j-1}\\
&= \mu \left(f_{j-1}+\frac{\gamma}{\gamma_0}\left(f_j-f_{j-1}\right)\right),
\end{align*}
and 
\begin{align*}
\sum_{i=1}^k X_i &= \sum_{i=1}^n X_i - \sum_{i=k+1}^n X_i\\
&=\mu \left(f_{j-1}+\frac{\gamma}{\gamma_0}\left(f_j-f_{j-1}\right)-\ell\right).
\end{align*}
The expected value of $Y_{1,k}$ is 
\begin{displaymath}
\mathbb{E}\left[Y_{1,k}\right]=\frac{1}{\mu}\sum_{i=1}^k X_i 
= (f_{j-1}-\ell) +\frac{\gamma}{\gamma_0}\left(f_j-f_{j-1}\right).
\end{displaymath}
Due to the assumption $p_t(i)\geq p_{\min}=1/4$,
the variance of $Y_{1,k}$ is 
\begin{align*}
  \var{Y_{1,k}} &=\sum_{i=1}^k p_t(i)(1-p_t(i))\\
&\geq p_{\min}(k-\expect{Y_{1,k}})\\
  &= \frac{1}{4}\left(n-\ell-\mathbb{E}\left[Y_{1,k}\right]\right) \\
&= \frac{1}{4}\left(n-\ell-f_{j-1}-\frac{\gamma}{\gamma_0}\left(f_j-f_{j-1}\right)+\ell\right)\\
&\geq \frac{1}{4}\left(n-f_{j-1}- d\left(n-f_{j-1}\right)\right)\\
&=\frac{1}{4}\left(n-f_{j-1}\right)\left(1-d\right).
\end{align*}
The probability of sampling an offspring in
$A_{\geq j+1}$ is bounded
from below by
\begin{displaymath}
\Pr\left(Y_{1,n} \geq f_{j}\right) 
\geq \Pr(Y_{1,k}\geq f_j-\ell)\cdot \Pr(Y_{k+1,n} = \ell),
\end{displaymath}
where 
\begin{displaymath}
\Pr(Y_{k+1,n}=\ell)=\left(1-\frac{1}{n}\right)^{\ell}\geq \left(1-\frac{1}{n}\right)^{n-1}\geq \frac{1}{e},
\end{displaymath}
and 
\begin{align}\label{upgrade-prob-g2}
\begin{split}
&\Pr\left(Y_{1,k}\geq f_j-\ell\right) \\
&\ge \Pr\left(Y_{1,k}\geq \mathbb{E}\left[Y_{1,k}\right]\right)
	-\Pr\left(\mathbb{E}\left[Y_{1,k}\right]\leq Y_{1,k}\leq f_j-\ell\right).
\end{split}
\end{align}
By 
Theorem~\ref{thm:anticoncentration}, we have
\begin{align*}
\Pr(\mathbb{E}\left[Y_{1,k}\right]\leq Y_{1,k}\leq f_j-\ell) 
&\leq \frac{\eta \left(f_j-\ell-\mathbb{E}[Y_{1,k}]\right)}{\sqrt{\var{Y_{1,k}}}}\\
&= \eta \left(1-\frac{\gamma}{\gamma_0}\right)\frac{f_j-f_{j-1}}{\sqrt{\var{Y_{1,k}}}}\\
&= 2\eta \left(1-\frac{\gamma}{\gamma_0}\right)\frac{d}{\sqrt{1-d}}\\
&\leq \left(1-\frac{\gamma}{\gamma_0}\right)\frac{d}{\sqrt{1-d}}.
\end{align*}
The last inequality follows from $\eta \approx 0.4688 <1/2$.
Note that
$\Pr\left(Y_{1,k}\geq \mathbb{E}\left[Y_{1,k}\right]\right)
\ge \psi = \Omega(1)$ due to 
Lemma~\ref{lower-bound-on-Y-k}, so
\eqref{upgrade-prob-g2} becomes
\begin{equation}\label{lower-bound-on-Y-1-k}
\Pr(Y_{1,k}\geq f_j-\ell) 
\geq \psi - \left(1-\frac{\gamma}{\gamma_0}\right)\frac{d}
{\sqrt{1-d}}
\geq \psi\frac{\gamma}{\gamma_0}.
\end{equation}
The last inequality is satisfied if for any $j\in [m-1]$, 
\begin{align*}
\frac{d}{\sqrt{1-d}}\leq \psi
&\iff \psi^{-2} d^2+d-1 \leq 0.
\end{align*}
The discriminant of this quadratic equation is 
$\Delta = 1+4\psi^{-2}>0$.
Vieta's formula \cite{algebravol1}
yields that the product of its two solutions is negative, implying that
the equation has two real solutions 
$d_1<0$ and $d_2>0$. Specifically,
\begin{displaymath}
d_{1} = -(1 + \sqrt{\Delta})\psi^{2}/2 <0,
\end{displaymath}
and
\begin{displaymath}
d_{2} = (-1+ \sqrt{\Delta})\psi^{2}/2\in (0,1).
\end{displaymath}
Therefore, if we choose any value of $d$ such that
$0<d \leq d_2$, then inequality (\ref{lower-bound-on-Y-1-k}) 
always holds. 
The probability of sampling an offspring 
in $A_{\geq j+1}$ 
is therefore bounded from below by
\begin{displaymath}
\Pr(Y_{1,n}\geq f_j)\geq \frac{1}{e}\cdot \psi\frac{\gamma}{\gamma_0} 
\geq (1+\delta)\gamma.
\end{displaymath}
The last inequality holds if we choose the population size 
in the \umda such that
$\mu/\lambda = \gamma_0 \leq \psi/(1+\delta)e$, where 
$\delta \in (0,1]$. Condition
(G2) then follows.

\textbf{Step 3}:
Assume that 
$|P_t\cap A_{\geq j}|\geq \gamma_0\lambda=\mu$. 
This means that the $\mu$ fittest 
individuals in the current sorted population $P_t$ belong to 
levels $A_{\ge j}$. In other words,
\begin{displaymath}
\sum_{i=1}^n X_i\geq \mu f_{j-1}, 
\end{displaymath}
and 
\begin{displaymath}
\sum_{i=1}^k X_i = \sum_{i=1}^n X_i-\sum_{i=k+1}^n X_i
\geq \mu f_{j-1}-\mu \ell =\mu (f_{j-1}-\ell).
\end{displaymath}
The expected value of $Y_{1,n}$ is 
\begin{equation}\label{expectation-of-Y-n}
\begin{split}
\mathbb{E}\left[Y_{1,n}\right] &=\sum_{i=1}^n p_t(i) \\ 
&= \frac{1}{\mu}\sum_{i=1}^{k}X_i +\sum_{i=k+1}^n \left(1-\frac{1}{n}\right) \\
&\geq f_{j-1}-\frac{\ell}{n}.
\end{split}
\end{equation}
An individual belonging to the higher levels 
$A_{\geq j+1}$ must have at least $f_j$ ones. 
The probability of sampling 
an offspring $y \in A_{\geq j+1}$ is equivalent to $\Pr(Y_{1,n}\geq f_j)$. According to
the level definitions and following the result of Lemma~\ref{inequality-g1}, we have
\begin{align*}
\Pr\left(Y_{1,n}\geq f_j\right) 
&= \Pr\left(Y_{1,n}\geq f_{j-1}+\lceil d\sqrt{n-f_{j-1}}\rceil\right)\\
&\geq \Pr\left(Y_{1,n}\geq \mathbb{E}\left[Y_{1,n}\right]+d\sqrt{n-\mathbb{E}\left[Y_{1,n}\right]}\right).
\end{align*}

In order to obtain a lower bound on 
$\Pr\left(Y_{1,n}\geq f_j\right)$, we need to bound 
the probability 
$\Pr(Y_{1,n}\geq \mathbb{E}\left[Y_{1,n}\right]+d\sqrt{n-\mathbb{E}\left[Y_{1,n}\right]})$
from below by a constant. We  obtain such
a bound by applying the result of
Lemma~\ref{ce-lemma}.
This lemma with constant $d^*\geq 1/p_{\min}=4$ 
and $d \le d^*$ yields
\begin{align*}
&\Pr\left(Y_{1,n}\geq f_j\right)\\
&\geq \Pr\left(Y_{1,n}\geq \mathbb{E}\left[Y_{1,n}\right]+d\sqrt{n-\mathbb{E}\left[Y_{1,n}\right]}\right)\\
&\geq \Pr\left(Y_{1,n}\geq \min\bigg\{\mathbb{E}\left[Y_{1,n}\right]+d^*\sqrt{n-\lfloor \mathbb{E}\left[Y_{1,n}\right]\rfloor},n\bigg\}\right)\\
&\geq \kappa>0.
\end{align*}
Hence, the probability of sampling an offspring in levels 
$A_{\geq j+1}$ is bounded from below by a positive constant $z_j:=\kappa$ independent of $n$.

\textbf{Step 4}:
We consider condition (G3) regarding the 
    population size. We have $1/\delta^2=\mathcal{O}(1)$,
    $1/z_*= \mathcal{O}(1)$, and $m=\mathcal{O}(\sqrt{n})$. 
    Therefore, there must 
    exist a constant $c>0$ such that	    
    \begin{align*}
      \left(\frac{c}{\gamma_0}\right)\sqrt{n}\ln(n) 
      \geq \left(\frac{4}{\gamma_0\delta^2}\right)\ln\left(\frac{128m}{z_*\delta^2}\right). 
    \end{align*}
    The requirement $\mu\geq c\sqrt{n}\ln(n)$ now implies that
    \begin{align*}
      \lambda 
      =  \frac{\mu}{\mu/\lambda} 
      \geq \left(\frac{c}{\gamma_0}\right)\sqrt{n}\ln(n)
      \geq \left(\frac{4}{\gamma_0\delta^2}\right)\ln\left(\frac{128m}{z_*\delta^2}\right);
    \end{align*}
    hence, condition (G3) is satisfied. 
    
\textbf{Step 5}:
The probability of sampling an offspring in
levels $A_{\geq j+1}$ 
is bounded from below by $z_j=\kappa$. 
Having satisfied all three conditions, 
Theorem~\ref{thm:levelbasedtheorem} 
then guarantees an
upper bound on the expected optimisation 
time of the \umda 
on \onemax, assuming that $\mu=\Omega(\sqrt{n}\log n)$,
\begin{displaymath}
\expect{T} 
=\mathcal{O}\left(\lambda\sum_{j=1}^m \frac{1}{z_j}+\sum_{j=1}^m \frac{1}{z_j}\right)
=\mathcal{O}(m\lambda) = \mathcal{O}\left(\lambda\sqrt{n}\right)
\end{displaymath}
since $m=\Theta(\sqrt{n})$ due to Lemma~\ref{lem:no-of-levels}. 
\end{proof}

\section{Empirical results}\label{sec:empirical}
We have proved upper
bounds on the expected optimisation time 
of the \umda  on  \onemax, \leadingones and \binval. 
However, they are only asymptotic upper bounds 
as growth functions of the problem and population
sizes. They provide no information on 
the multiplicative constants or the 
influences of lower order terms. 
Our goal is also
to investigate the runtime behaviour 
for larger populations.
To complement the theoretical findings, we therefore
carried out some 
experiments by running the \umda on 
the three functions. 

For each function, the parameters  were chosen 
consistently with the theoretical analyses. 
Specifically, we set $\lambda=n$, 
and $n \in \{100, 200,\ldots,4500\}$. Although the
theoretical results imply that significantly smaller population sizes
would suffice,
e.g. $\lambda=O(\log n)$ for Theorem \ref{onemax-small-mu}
we chose a larger population size in the experiments to more easily
observe the impact of $\lambda$
on the running time of the algorithm.
The results are shown in 
Figures~\ref{fig:exp-onemax}--\ref{fig:exp-binval}.
For each value of $n$, the algorithm is run $100$ 
times, and then the average runtime is computed. 
The mean runtime for each value of $n$ is estimated
with $95\%$ confidence intervals using the
\textit{bootstrap percentile method} \cite{bib:Lehre2014} 
with $100$ bootstrap samples. Each mean point is plotted 
with two error bars to illustrate the upper and 
lower margins of the confidence intervals.

\subsection{OneMax}
In Section~\ref{sec:umda-onemax}, we 
obtained two upper bounds on the expected optimisation time of the 
the \umda on  \onemax, which 
are tighter than the earlier bound
$\mathcal{O}(n\lambda\log \lambda)$ 
in \cite{bib:Dang2015a}, as follows
\begin{itemize}
\item $\mathcal{O}\left(\lambda n\right)$ when 
$\mu=\Omega(\log n)\cap \mathcal{O}(\sqrt{n})$,
\item $\mathcal{O}(\lambda\sqrt{n})$ when
$\mu = \Omega(\sqrt{n}\log(n))$.
\end{itemize}

\begin{figure}[htp] 
    \centering
    \subfloat[small $\mu$]{%
        \includegraphics[width=0.45\textwidth]{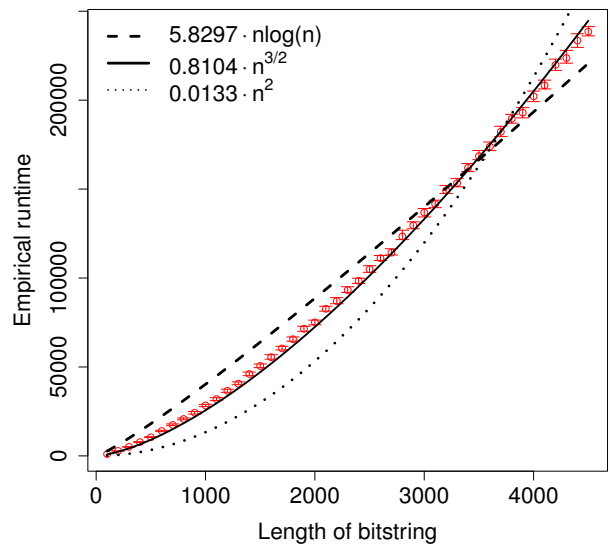}%
        \label{om-small-mu}%
        }%
    \hfill%
    \subfloat[large $\mu$]{%
        \includegraphics[width=0.45\textwidth]{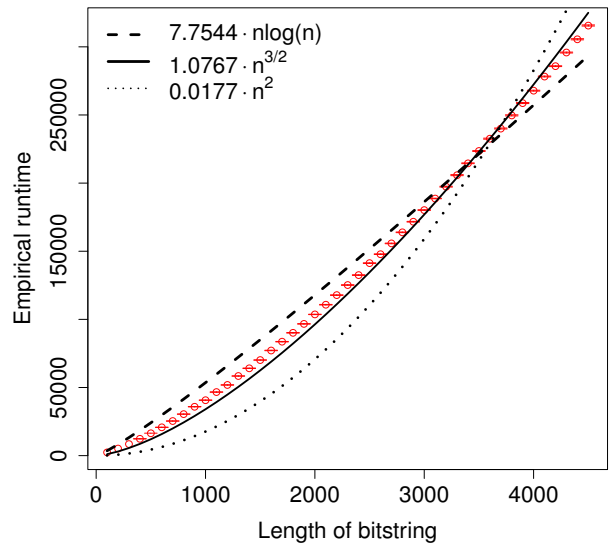}%
        \label{om-large-mu}%
        }%
    \caption{Mean runtime of the \umda on \onemax with 95\% 
		confidence intervals plotted with error bars in red colour. 
		Models are also fitted via non-linear regression.}
    \label{fig:exp-onemax}
\end{figure}

We therefore experimented with two different settings 
for the parent population size: 
$\mu=\sqrt{n}$ and $\mu=\sqrt{n}\log(n)$. 
We call the first setting small population
and the other large population.
The empirical runtimes are shown 
in Figure~\ref{fig:exp-onemax}.
Theorem~\ref{onemax-small-mu} implies the upper bounds
$\mathcal{O}(n^2)$ for the setting of small population and 
$\mathcal{O}(n^{3/2})$ for the setting of large population.
Following \cite{bib:Lehre2014}, we  
identify the three positive constants 
$c_1, c_2$ and $c_3$ that best fit the models $c_1 n\log n$, $c_2 n^{3/2}$ and
$c_3n^{2}$  in non-linear least square regression. Note in particular that 
these models
were chosen because they are close to the theoretical results. The 
correlation coefficient $\ccoef$ is then calculated for each model 
to find the best-fit model.

\begin{table}[h]
	\centering
	\caption{Correlation coefficient $\ccoef$  for the best-fit models 
	                  in the experiments with \onemax shown in 
	                  Figures~\ref{om-small-mu} and \ref{om-large-mu}.}
	\label{tab:onemax}
	\begin{tabular}{@{}lll@{}}
	\toprule
	   \textbf{Setting}
	  &\textbf{Model}
	  &$\boldsymbol{\ccoef}$\\
	\midrule
	   $\mu = \sqrt{n}$
	  &$5.8297 \;n\log n$ & $0.9968$\\ 
	  &$0.8104 \;n^{3/2}$ & $\mathbf{0.9996}$\\
	  &$0.0133 \;n^{2}$   & $0.9910$\\
	\midrule
	   $\mu = \sqrt{n}\log n$
	  &$7.7544 \;n\log n$ & $0.9974$\\
	  &$1.0767 \;n^{3/2}$ & $\mathbf{0.9995}$\\
	  &$0.0177 \;n^{2}$   & $0.9903$\\
	\bottomrule
	\end{tabular}
\end{table}	

In Table \ref{tab:onemax}, we observe that
for small parent populations (i.e. $\mu = \sqrt{n}$), 
model  $0.8104\;n^{3/2}$
fits 
the empirical data best,
while the quadratic model gives the 
worst
result.
 For larger parent population (i.e. $\mu = \sqrt{n}\log n$), the model
$1.0767~n^{3/2}$ fits best the empirical data among the three models. 
Since $0.8104~n^{3/2} \in \mathcal{O}(n^2)$,
these findings are consistent with the
theoretical expected optimisation time and may further suggest that
the quadratic bound in case of small population is not tight.

\subsection{LeadingOnes}
We conducted experiments with $\mu=\sqrt{n}$, 
and $\lambda=n$. According to 
Theorem~\ref{thm:leadingones}, 
the upper bound of
the expected runtime is in this case 
$\mathcal{O}(n\lambda\log \lambda+ n^2) = \mathcal{O}(n^2\log n)$. 
Figure~\ref{fig:exp-leadingones} shows the empirical runtime.
Similarly to the \onemax problem,
we fit the empirical runtime with four different models -- 
$c_1n\log n$, $c_2n^{3/2}$, $c_3n^2$ and $c_4n^2\log n$ -- 
using non-linear regression. 
The best values of the four constants 
are shown in Table~\ref{tab:leadingones}
along with the correlation coefficients of the models.

\begin{figure}[h] 
    \centering      
    \includegraphics[width=0.5\textwidth]{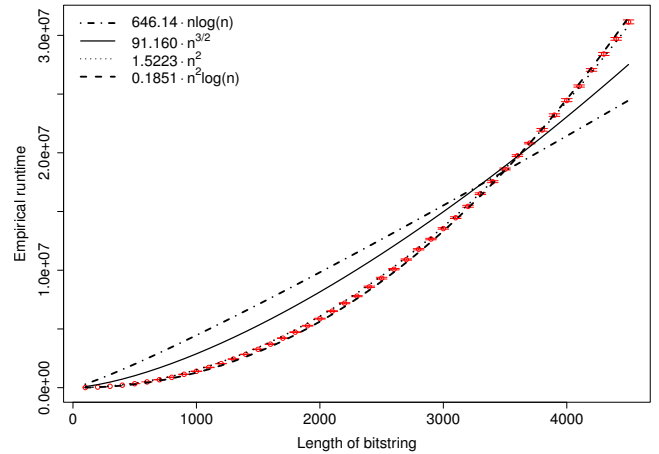}%
    \caption{Mean runtime of the \umda on \leadingones with 95\% 
		confidence intervals plotted with error bars in red colour. 
		Models are also fitted via non-linear regression.}
    \label{fig:exp-leadingones}
\end{figure}

\begin{table}[h]
	\centering
	\caption{Correlation coefficient $\ccoef$ for the best-fit models 
	         in the experiments with \leadingones shown in 
	         Figure~\ref{fig:exp-leadingones}.}
	\label{tab:leadingones}
	\begin{tabular}{@{}lll@{}}
	\toprule
	   \textbf{Setting}
	  &\textbf{Model}
	  &$\boldsymbol{\ccoef}$\\
	\midrule
	    $\mu = \sqrt{n}$
	   &$646.14 \;n\log n$ & $0.9756$\\
	   &$91.160 \;n^{3/2}$ & $0.9928$\\
       &$1.5223 \;n^{2}$   & $\mathbf{0.9999}$\\
       &$0.1851 \;n^{2}\log n$   & $\mathbf{0.9999}$\\
	\bottomrule
	\end{tabular}
\end{table}	

Figure~\ref{fig:exp-leadingones} and 
Table~\ref{tab:leadingones}  show 
that both the model $1.5223~n^2$ 
and the model $0.1851~n^2\log n$, having the same correlation coefficient,
 fit well with the empirical data (\ie 
 the empirical data lie between these two curves).
 This finding
is consistent with 
the
theoretical runtime 
bound
$\mathcal{O}(n^2\log n)$. Note also that these two models 
 differ asymptotically by $\Theta(\log n)$, suggesting that 
our analysis of the \umda 
on \leadingones is nearly tight.

\subsection{BinVal}
Finally, we consider \binval.
The upper bound
$\mathcal{O}(n\lambda\log \lambda + n^2)$ 
from Theorem~\ref{thm:leadingones} for the function is 
identical to the bound for \leadingones.
Since \binval is also a linear function like \onemax,
we decided to set the experiments similarly for these functions,
\ie with different parent populations
$\mu=\sqrt{n}$ and $\mu=\sqrt{n}\log n$.
The empirical results are shown in Figure~\ref{fig:exp-binval}. 
Again the empirical runtime  is fitted to the three models 
$c_1 n\log n$, $c_2 n^{3/2}$ and $c_3 n^{2}$. 
The best values of 
$c_1, c_2$ and $c_3$ are listed in Table~\ref{tab:binval}, along with the
correlation coefficient for each model.

\begin{figure}[h] 
    \centering
    \subfloat[small $\mu$]{%
        \includegraphics[width=0.45\textwidth]{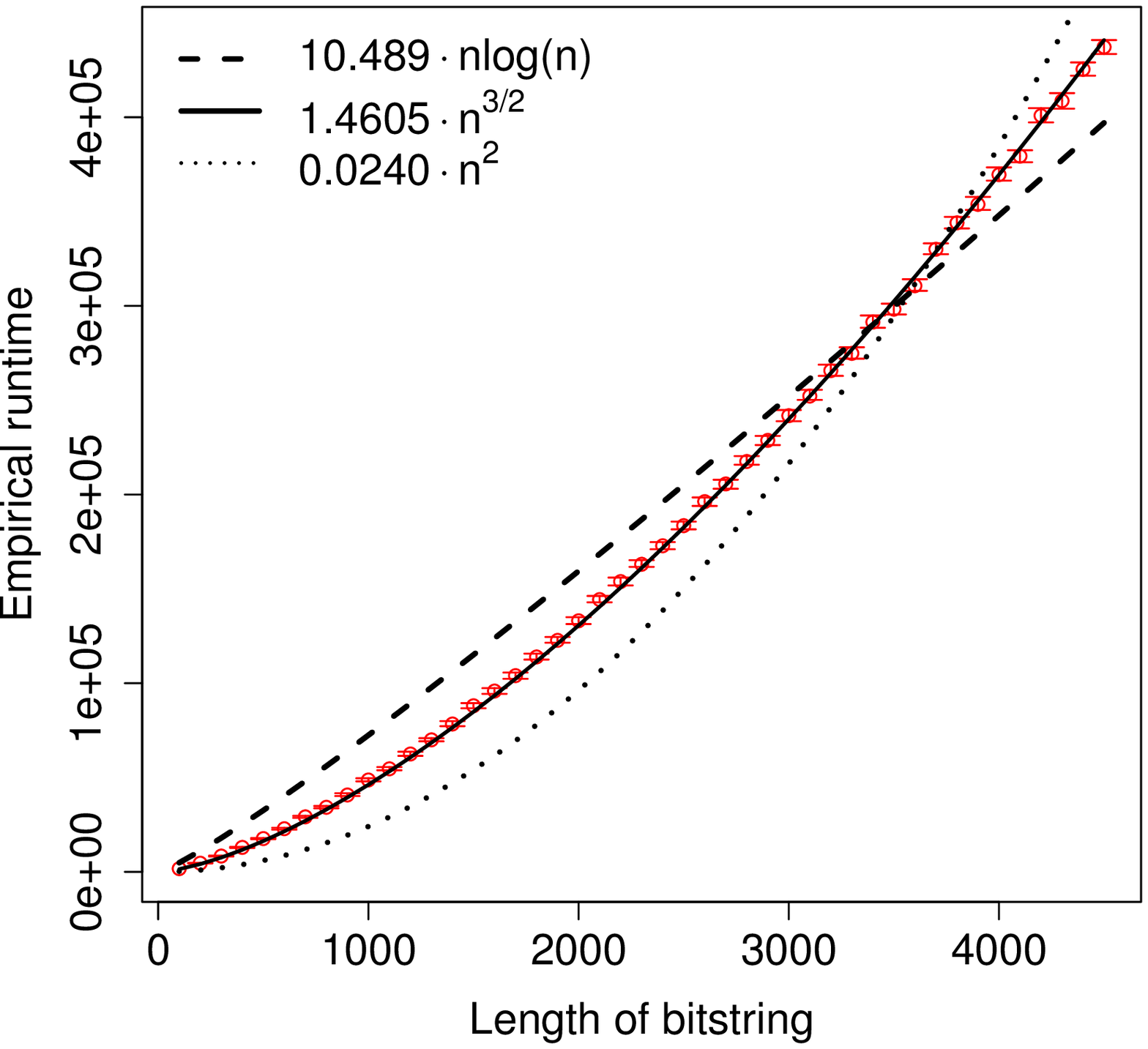}%
        \label{binval-small-mu}%
        }%
    \hfill%
    \subfloat[large $\mu$]{%
        \includegraphics[width=0.45\textwidth]{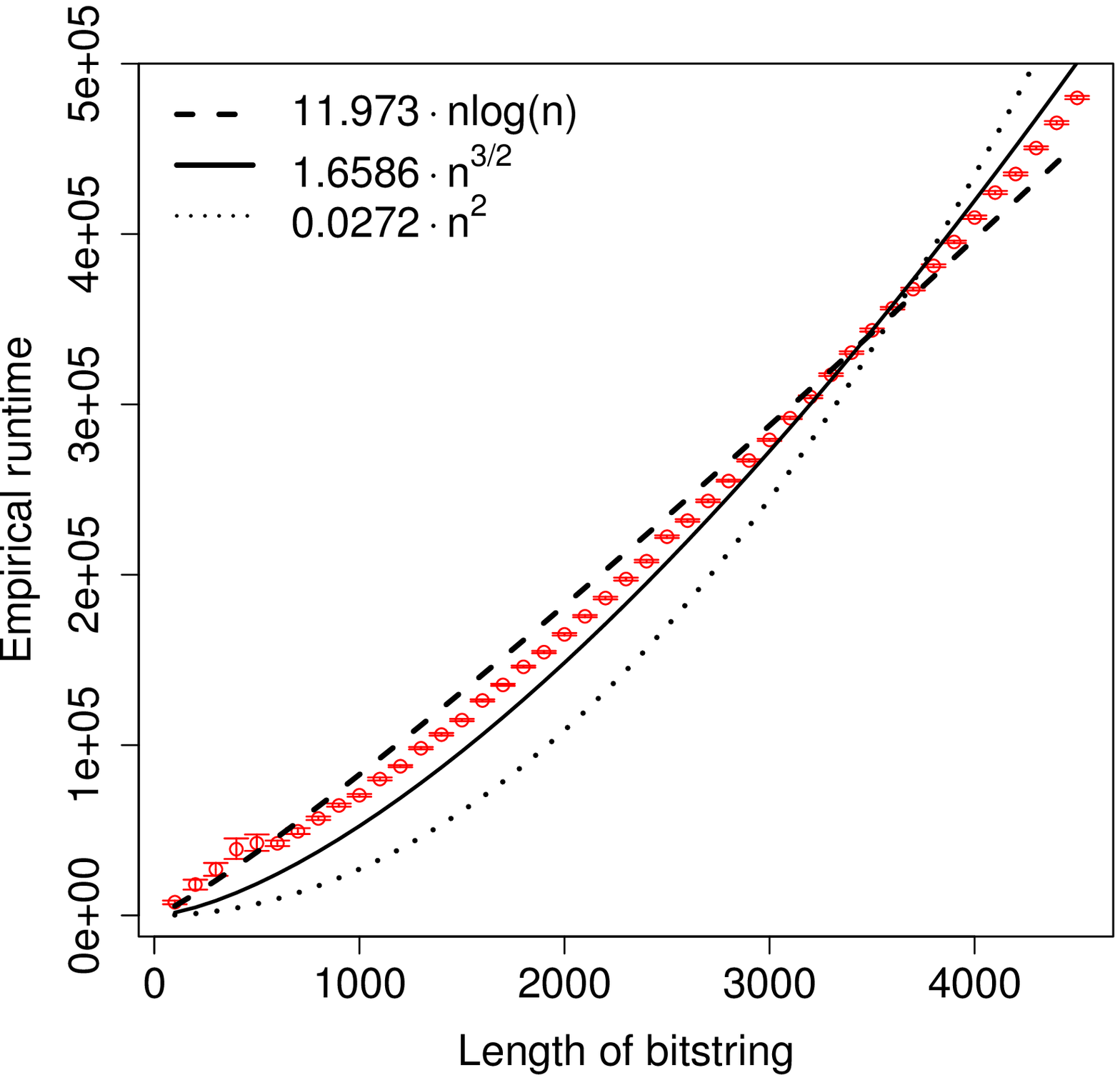}%
        \label{binval-large-mu}%
        }%
    \caption{Mean runtime of the \umda on \binval with 95\% 
		confidence intervals plotted with error bars in red colour. 
		Models are also fitted via non-linear regression.}
    \label{fig:exp-binval}
\end{figure}

\begin{table}[h]
	\centering
	\caption{Correlation coefficient $\ccoef$  for the best-fit models 
	        in the experiments with \binval shown in 
	         Figures~\ref{binval-small-mu} and \ref{binval-large-mu}}
	\label{tab:binval}
	\begin{tabular}{@{}lll@{}}
	\toprule
	   \textbf{Setting}
	  &\textbf{Model}
	  &$\boldsymbol{\ccoef}$\\
	\midrule
	   $\mu = \sqrt{n}$
      &$10.489 \;n\log n$ & $0.9952$\\
      &$1.4605 \;n^{3/2}$ & $\mathbf{0.9999}$\\
      &$0.0240 \;n^{2}$   & $0.9933$\\
	\midrule
	   $\mu = \sqrt{n}\log n$
      &$11.973 \;n\log n$ & $0.9972$\\
      &$1.6596 \;n^{3/2}$ & $\mathbf{0.9994}$\\
      &$0.0272 \;n^{2}$   & $0.9903$\\
	\bottomrule
	\end{tabular}
\end{table}	

Theorem~\ref{thm:leadingones}  
gives the  upper bound of
$\mathcal{O}(n^2\log n)$
for
the expected runtime 
of \binval.
However,
Figure~\ref{fig:exp-binval} and Table~\ref{tab:binval} 
 show clearly that 
the model $1.4605 \;n^{3/2}$ fits 
best
the empirical runtime for $\mu=\sqrt{n}$.
On the other hand,
the empirical runtime lies between the two models
$11.973 \;n\log n$ and $1.6586 \;n^{3/2}$
when $\mu=\sqrt{n}\log n$.
While these observations
are consistent with the 
theoretical upper bound since
$\mathcal{O}(n^{3/2})$ and $\mathcal{O}(n\log n)$ are all members of
$\mathcal{O}(n^2\log n)$, 
they also suggest that
our analysis of the \umda
on \binval given by Theorem~\ref{thm:leadingones} may be loose.

\section{Conclusion}\label{sec:conclusion}
Despite the popularity of \edas in real-world applications, 
little has been known 
about their theoretical optimisation time, even for
apparently simple settings such as the \umda on toy functions. 
More results for the \umda 
on these simple problems with well-understood structures
provide a way to describe and compare the performance 
of the algorithms with other search heuristics. Furthermore,
results about
the \umda are not only relevant to evolutionary computation, but also
to population genetics where it corresponds to the notion of
\emph{linkage equilibrium} \cite{bib:muhlenbein2002,bib:Slatkin2008}.

We have analysed the expected optimisation time of the
the \umda  on three benchmark 
problems: \onemax, \leadingones and \binval.
For both \leadingones and \binval, we proved the 
upper bound of
$\mathcal{O}(n\lambda\log \lambda+n^2)$, which holds for
$\lambda=\Omega(\log n)$. For \onemax, two upper bounds of
$\mathcal{O}(\lambda n)$ 
and $\mathcal{O}(\lambda \sqrt{n})$ were obtained for 
$\mu = \Omega(\log n)\cap \mathcal{O}(\sqrt{n})$ and 
$\mu= \Omega(\sqrt{n}\log n)$, respectively. 
Although our result assumes that $\lambda\geq (1+\beta)\mu$ for some
positive constant $\beta>0$, 
it no longer requires that $\lambda=\Theta(\mu)$ 
as  in \cite{bib:Witt2017}.  
Note that if
$\lambda =\Theta(\log n)$, a tight bound
of $\Theta(n\log n)$ on the 
expected optimisation time of the \umda on
\onemax is obtained, matching the well-known tight 
bound of $\Theta(n\log n)$
for the (1+1) \ea on the class of linear functions. 
Although we did not obtain a runtime bound 
when the parent population size is 
$\mu = \Omega(\sqrt{n})\cap \mathcal{O}(\sqrt{n}\log n)$, our results 
finally close the existing $\Theta(\log \log n)$-gap
between the first upper bound of
$\mathcal{O}(n\log n \log \log n)$ for 
$\lambda=\Omega(\mu)$  \cite{bib:Dang2015a} 
and the relatively new lower
bound of $\Omega(\mu\sqrt{n}+n\log n)$ 
for $\lambda = (1+\Theta(1))\mu$
\cite{bib:Krejca}. 

Our analysis further demonstrates that the level-based theorem can
yield, relatively easily, asymptotically tight 
upper bounds for non-trivial,
population-based algorithms. An important additional component of the
analysis was the use of anti-concentration properties of the
Poisson-Binomial distribution. Unless the variance of the sampled
individuals is not too small, the distribution of the population
cannot be too concentrated anywhere, even around the mean, 
yielding sufficient diversity to
discover better solutions. We expect that similar arguments will lead to
new results in runtime analysis of evolutionary algorithms.

\appendix
\section{Appendix}

\begin{lemma}[\cite{mitrinovic1970analytic}]
	\label{exponential-bounds}
	For all $t\in \mathbb{R}$ and $n\in \mathbb{R}^+$, 
    \begin{displaymath}
    \left(1+\frac{t}{n}\right)^n \leq e^t \leq \left(1+\frac{t}{n}\right)^{n+t/2}.
    \end{displaymath}
\end{lemma}

\begin{lemma}[Theorem 3.2, \cite{bib:Jogdeo}]\label{int-expectation}
	Let $Y_1,Y_2,\ldots,Y_n$ be $n$ 
    independent Bernoulli random variables, and 
    $Y\coloneqq \sum_{i=1}^{n}Y_i$ is 
    the sum of these random variables. 
    If $\mathbb{E}[Y]$ is an integer, then 
	\begin{displaymath}
	\Pr\left(Y\geq \mathbb{E}[Y]\right)\geq 1/2.
	\end{displaymath}
\end{lemma}

\begin{lemma}[Stirling's approximation \cite{bib:LCRC}]
\label{stirling} For all $n\in \mathbb{N}$,
	\begin{displaymath}
		n!=\Theta\left(\frac{n^{n+1/2}}{e^n}\right).
	\end{displaymath}
\end{lemma}

In the following we write $X \preceq Y$ to denote that random variable $Y$ 
stochastically dominates random variable 
$X$, \ie $\prob{X \geq k} \leq \prob{Y \geq k}$ for all $k \in 
\mathbb{R}$. The lemma below can be easily proved with coupling argument \cite{bib:Roch2015}.

\begin{lemma}\label{lem:stocdom}
    Let $X_1,X_2, Y_1$ and $Y_2$ be 
    independent random variables such that 
    $X_1 \succeq Y_1$ and $X_2 \succeq Y_2$. Then
    $X_1 +X_2 \succeq Y_1+Y_2$.    
\end{lemma}
\begin{proof} The proof is taken from Corollary~4.27 in \cite{bib:Roch2015}. 
Let $(\hat{X}_1,\hat{Y}_1)$ and $(\hat{X}_2,\hat{Y}_2)$ 
be independent, monotone couplings of $(X_1,Y_1)$
and $(X_2,Y_2)$ on the same probability space. It then holds that
$X_1+X_2 \sim \hat{X}_1+\hat{X}_2 \succeq \hat{Y}_1 + \hat{Y}_2 \sim Y_1+Y_2$.
\end{proof}

\begin{lemma}[Lemma~3, \cite{bib:Wu2017}]\label{ce-lemma}
Let $Y_1,\ldots,Y_n$ be $n$ 
    independent Bernoulli random variables with success 
    probabilities $p_1, \ldots, p_n$. Let 
    $Y:= \sum_{i=1}^n Y_i$ be the sum of these variables.
	If $p_i\geq p_{\min}$ for all $i\in [n]$, where 
    $p_{\min}>0$ is a constant, 
	and any constant $d^*\geq 1/p_{\min}$ then
	\begin{displaymath}
	\Pr\left(Y \geq \min\bigg\{\mathbb{E}\left[Y\right]+d^*\sqrt{n-
		\lfloor \mathbb{E}[Y]\rfloor},n\bigg\}\right)\geq \kappa,
	\end{displaymath}
	where $\kappa$ is a positive constant independent of $n$. 
\end{lemma}

\begin{lemma}\label{lem:no-of-levels}
	For any $n \in \mathbb{N}$, any constant $d \in (0,1]$ independent to $n$ and 
	the sequence $(f_i)_{i \in \mathbb{N}}$ defined according to \eqref{eqn:recur-levels},
	it holds that 
	\begin{description}
		\item[(i)]  $f_i \leq n$ for all $i \in \mathbb{N}$, and $\exists j \in \mathbb{N} \colon f_j = n$,
		\item[(ii)] if $\ell = \min \{i \in \mathbb{N} \mid f_i = n\}$ then 
        $\ell = \Theta(\sqrt{n})$.
	\end{description}
\end{lemma}

\begin{proof}
	We first prove (i), it is easy to see that $f_i$ 
    are all integer, \ie $f_i \in \mathbb{N}$ for 
    all $i \in \mathbb{N}$. Due to the ceiling 
    function if $f_i<n$, then $f_{i+1} \geq f_i + 1$, 
    in other words starting with $f_0 = 0$, 
    the sequence will increase steadily until 
    it hits $n$ exactly or overshoots it. 
    Assuming the later case of overshooting, that is, 
	$\exists k \geq 0 \colon f_k \leq n-1$ and 
    $f_{k+1} \geq n+1$ (and after that $f_{k+2}, \dots$ are ill-defined).
	By the definition of the sequence, the 
    property $1 + x > \lceil x \rceil$ of the 
    ceiling function and $d\leq 1$, we have 
	$$
	1 + \sqrt{n - f_k} > \lceil \sqrt{n - f_k}\rceil \geq \lceil d \sqrt{n - f_k}\rceil = f_{k+1} - f_k \geq 2,
    $$ 
	this implies $f_k < n - 1$ or $f_k \leq n-2$. 
    Repeating the above argument again gives that 
    $1 + \sqrt{n - f_k}  > 3$, and $f_k < n - 4$, 
    after a finite number of repetitions we will 
    conclude that $f_k<0$ which is a contradiction. 
    Therefore, the sequence must hit $n$ exactly at 
    one point in time then it will remain at that value.
	
	
	To bound $\ell$ in (ii), we pair $(f_i)_{i \in \mathbb{N}}$ 
    with $(r_i := \sqrt{n - f_i})_{i \in \mathbb{N}}$; 
    thus, this sequence starts at $r_0=\sqrt{n}$, 
    then decreases and eventually hits $0$, 
    that is, $\sqrt{n} = r_0 > r_1 > r_2 > \dots >r_{\ell-1} > r_{\ell} = 0$. 
    From \eqref{eqn:recur-levels}, we have
	\begin{align*}
	(r_{i} - r_{i+1})(r_{i} + r_{i+1})
	&= r_{i}^2 - r_{i+1}^2 
	= f_{i+1} - f_{i} 
	= \lceil d r_i \rceil,
	\end{align*}
	note that $1 + d r_i > \lceil d r_i \rceil \geq d r_i$, then for $i \leq \ell - 1$, we can divide both sides by $r_i + r_{i+1} > 0$ to get
	\begin{align*}
	\frac{1 + d r_i}{ r_i + r_{i+1}}
	> r_{i} - r_{i+1} 
	\geq \frac{d r_i}{ r_i + r_{i+1}}.
	\end{align*}
	Always restricted to $i \leq \ell - 1$, we have that $1>r_{i+1}/r_{i}\geq 0$, and therefore $d r_i / (r_i + r_{i+1}) = d/(1 + r_{i+1}/r_{i}) > d/2$. In addition, $f_i \leq n-1$ then $r_i = \sqrt{n - f_i} \geq 1$ or $1/r_i \leq 1$, so $(1 + d r_i) / (r_i + r_{i+1}) = (1/r_i + d)/(1 + r_{i+1}/r_{i}) \leq d + 1$. Therefore, for all $i \leq \ell - 1$
	\begin{align*}
	d+1
	> r_{i} - r_{i+1} 
	> \frac{d}{2}.
	\end{align*}
	Summing all these terms gives that 
	\begin{align*}
	\ell (d+1)
	> \sum_{i=0}^{\ell-1} (r_i - r_{i+1}) 
	= r_0 - r_\ell 
	= \sqrt{n} 
	> \frac{\ell d}{2},
	\end{align*}
	and this implies $2\sqrt{n}/d > \ell > \sqrt{n}/(d+1)$, or $\ell = \Theta(\sqrt{n})$.
\end{proof}

\begin{lemma}\label{lower-bound-on-Y-k}
	Let $Y_1, Y_2, \ldots,Y_k$ be $k$ ($k\geq 1$) 
	independent Bernoulli random variables with success probabilities 
	$p_1, p_2, \ldots, p_k$, where $p_i\geq p_{\min}=1/4$ for each $i\in [k]$. 
	Let $Y_{1,k}:=\sum_{i=1}^k Y_i$. Then we always have
	\begin{displaymath}
	\Pr\left(Y_{1,k}\geq \mathbb{E}\left[Y_{1,k}\right]\right) 
	\geq \Omega(1).
	\end{displaymath}
\end{lemma}

\begin{proof}
	We start by considering small values of $k$.
	If $k=1$, then 
	\begin{displaymath}
	\Pr\left(Y_{1,1}\geq \mathbb{E}\left[Y_{1,1}\right]\right)
	=\Pr(Y_{1}=1)=p_1 \geq 1/4.
	\end{displaymath}
	If $k=2$, then 
	$$
	\Pr\left(Y_{1,2}\geq \mathbb{E}[Y_{1,2}]\right) 
	\geq \Pr\left(Y_1=1\right)\cdot \Pr\left(Y_2=1\right)\geq p_1p_2 \geq (1/4)^2.
	$$
	For larger values of $k$, following \cite{bib:Wu2017}
	we introduce another 
	random variable $Z= \left(Z_1, \ldots, Z_k\right)$ with 
	success probabilities $z_1,\ldots,z_k$, where 
	$z_i\geq p_{\min}$ and 
	$\mathbb{E}[Z_{1,k}] = \sum_{i=1}^kz_i=\sum_{i=1}^k p_i=\mathbb{E}\left[Y_{1,k}\right]$.
	However, we shift the total weight $\mathbb{E}\left[Y_{1,k}\right]$ as far as possible to the $Z_i$
	with smaller indices as follows. We define
	$m=\floor{\frac{\mathbb{E}[Y_{1,k}]-kp_{\min}}{1-p_{\min}}}$,
	and let $Z_1,\ldots,Z_m$ all get success probability 1, and $Z_{m+2},\ldots,Z_k$ 
	get $z_i=p_{\min}$, more precisely
	\begin{displaymath}
	z_i = \begin{cases}
	~1 ,& \text{for  }i=1,\ldots,m,\\
	~q,& \text{for }i=m+1,\\
	~p_{\min} , & \text{for }i=m+2,\ldots,k,
	\end{cases}
	\end{displaymath}
	where $q \in [p_{\min},1]$. It is quite clear that 
	$(z_1,\ldots,z_k)$ majorises $\left(p_t(1),\ldots,p_t(k)\right)$. From 
	\cite{marshall1979inequalities,gleser1975}, we now have
	\begin{align*}
	\Pr\left(Y_{1,k}\geq \mathbb{E}\left[Y_{1,k}\right]\right)
	&\geq \Pr\left(Y_{1,k}\geq \mathbb{E}\left[Y_{1,k}\right]+1\right)\\
	&\geq \Pr(Z_{1,k}\geq \mathbb{E}\left[Z_{1,k}\right]+1).
	\end{align*}
	Furthermore, with probability $1$ we can get $m$ ones and 
	$$
	\mathbb{E}[Z_{m+2,k}]=\mathbb{E}\left[Z_{1,k}\right]-m-q \iff \mathbb{E}[Z_{m+2,k}]+q = \mathbb{E}\left[Z_{1,k}\right]-m,
	$$ then
	\begin{align*}
	\Pr(Z_{1,k}&\geq \mathbb{E}\left[Z_{1,k}\right]+1) \\
	&\geq \Pr(Z_{m+1,k} \geq \mathbb{E}\left[Z_{1,k}\right]+1-m)\\
	&\geq \Pr(Z_{m+1}=1)\cdot\Pr(Z_{m+2,k}\geq \mathbb{E}\left[Z_{1,k}\right]-m)\\
	&= q\cdot\Pr(Z_{m+2,k}\geq \mathbb{E}[Z_{m+2,k}]+q)\\
	&\geq p_{\min}\cdot\Pr(Z_{m+2,k}\geq \mathbb{E}[Z_{m+2,k}]+1).
	\end{align*}
	The last inequality follows the fact that $ p_{\min}\leq q\leq 1$. 
	We now need a lower bound on the
	probability $\Pr(Z_{m+2,k}\geq \mathbb{E}[Z_{m+2,k}]+1)$, 
	where 
	\begin{displaymath}
	    Z_{m+2,k} \sim \text{Bin}\left(k-m-1,\frac{1}{4}\right).
	\end{displaymath}
	Now let $k-m-1 = 4t+x = 4(t-1)+x+4$, where $t \in \mathbb{N}$ 
	and $x\in \{0,1,2,3\}$. Then $\mathbb{E}[Z_{m+2,k}]=t+\frac{x}{4}$, and 
	\begin{align*}
	&\Pr(Z_{m+2,k}\geq \mathbb{E}[Z_{m+2,k}]+1) \\
	&= \Pr(Z_{m+2,k}\geq t+\frac{x}{4}+1)\\
	&\geq \Pr\left(Z_{m+2,k}\geq 4(t-1)+x+4\right)\\
	&\geq \Pr(Z_{m+2,m+2+4(t-1)-1}\geq t-1)\\
    &\qquad \qquad \cdot \Pr(Z_{m+2+4(t-1),n}\geq x+4)\\
	&= \Pr(Z_{m+2,m+2+4(t-1)-1}\geq \mathbb{E}[Z_{m+2,m+2+4(t-1)-1}])\\
	&\qquad\qquad \cdot \Pr(Z_{m+2+4(t-1),n}\geq x+4)\\
	&\geq \frac{1}{2}\cdot \left(\frac{1}{4}\right)^{x+4} 
	\geq \frac{1}{2}\cdot \left(\frac{1}{4}\right)^7.
	\end{align*}
	The result follows Lemma~\ref{int-expectation},
	where $\mathbb{E}[Z_{m+2,m+2+4(t-1)-1}]$ is an integer,
	and $x\le 3$. 
\end{proof}

\begin{lemma}\label{inequality-g1}
	For any constant $d\leq 1$ and 
    $\mathbb{E}[Y_{1,n}]\ge f_{j-1}-\ell/n$, it holds that
	\begin{equation}\label{eq:inequality-g1}
	\mathbb{E}\left[Y_{1,n}\right]+d\sqrt{n-\mathbb{E}\left[Y_{1,n}\right]} \geq f_{j-1}+d\sqrt{n-f_{j-1}}.
	\end{equation}
\end{lemma}
\begin{proof}
	Let us rewrite (\ref{expectation-of-Y-n}) by introducing a variable $x\geq 0$ as follows:
	\begin{equation}\label{expectation}
	\mathbb{E}\left[Y_{1,n}\right]=f_{j-1}-\frac{\ell}{n}+x
	\end{equation}
	We consider two different cases.
	\begin{itemize}
		\item Case 1: If $x=\ell/n$, then $\mathbb{E}\left[Y_{1,n}\right]=f_{j-1}$, and 
		the lemma 
		holds for all values of $d$.
		\item Case 2: If $x\neq \ell/n$, then 
		substituting (\ref{expectation}) into \eqref{eq:inequality-g1} 
		and let $y:=x-\ell/n \in [-\ell/n,0)\cup (0,n-f_{j-1}]$, 
		we have
		\begin{equation*}\label{d-0-expression}
		d \leq g\left(y,f_{j-1}\right) := \frac{y}{\sqrt{n-f_{j-1}}-\sqrt{n-f_{j-1}-y}}
		\end{equation*}
		This always holds if we pick a constant 
		\begin{displaymath}
		        d \leq \min_{y,f_{j-1}}g\left(y,f_{j-1}\right).
		\end{displaymath}
		From $\partial g/\partial y=0$, we obtain $y=0$. Note that when $y=0$,  
		$\partial^2 g/\partial y^2 <0$. This means $g(y,f_{j-1})$ reaches the maximum 
		value when $y=0$ with respect to $f_{j-1}$, and 
		\begin{align*}
		&g\left(y,f_{j-1}\right) \\
		&\geq \min\big\{g\left(-\ell/n, f_{j-1}\right), g\left(n-f_{j-1}, f_{j-1}\right)\big\}\\
		&=\min\big\{\sqrt{n-f_{j-1}}, \sqrt{n-f_{j-1}}+\sqrt{n-f_{j-1}+\ell/n}\big\}\\
		&=\sqrt{n-f_{j-1}} \\
		&\geq \min_{f_{j-1}}\big\{\sqrt{n-f_{j-1}}\big\} \\
		&=1 
		\end{align*}
		due to $f_{j-1}\le n-1$. 
	\end{itemize}
	The lemma is proved by combining results of the two cases.
\end{proof}

\bibliographystyle{plain}     
\bibliography{references}
\end{document}

%% file: preamble.tex
\usepackage{amsfonts,amsmath,amssymb,amsthm}
\usepackage{xspace}
\usepackage{hyperref}
\usepackage{booktabs}
\usepackage{tikz}
\usepackage{graphicx}
\usepackage{enumitem}
\usepackage{mathtools}
\usepackage{tabulary}
\usepackage{tabularx}

\usepackage{threeparttable}

\usepackage{subfig}

\usepackage[linesnumbered,ruled,vlined]{algorithm2e}

\newtheorem{theorem}{Theorem} 
\newtheorem{lemma}      [theorem]{Lemma}

\newtheorem{definition} [theorem]{Definition}

\newcommand{\bigO}[1]{\mathcal{O}\left(#1\right)}

\newcommand{\ccoef}{\ensuremath{\rho}} 

\newcommand{\onemax}{\text{\sc OneMax}\xspace}           
\newcommand{\binval}{\text{\sc BinVal}\xspace}           
\newcommand{\leadingones}{\text{\sc LeadingOnes}\xspace} 
\newcommand{\bvleadingones}{\text{\sc BVLeadingOnes}\xspace} 

\newcommand{\umda}   {\text{UMDA}\xspace}            
\newcommand{\pbil}   {\text{PBIL}\xspace}            
\newcommand{\eda}    {\text{EDA}\xspace}             
\newcommand{\edas}   {\text{EDAs}\xspace}            
\newcommand{\ea}     {\text{EA}\xspace}              
\newcommand{\eas}    {\text{EAs}\xspace}             
\newcommand{\cga}    {\text{cGA}\xspace}             
\newcommand{\scga}   {\text{scGA}\xspace}            
\newcommand{\tmmasib}{\text{$2$-MMAS$_{\text{ib}}$}\xspace} 

\DeclarePairedDelimiter\floor{\lfloor}{\rfloor}

\DeclareMathOperator{\Ber}{Bernoulli}       

\newcommand{\prob}[1]{\Pr\left(#1\right)}
\newcommand{\expect}[1]{\mathbb{E}\left[#1\right]}
\newcommand{\var}[1]{\mathrm{Var}\left[#1\right]}

\usepackage{mathrsfs,dsfont}

\newcommand{\ab}{\hspace{0.125em}}                        
\newcommand{\ie}{\hbox{i.\ab e.}\xspace}                  

\usepackage{tikz}
\usetikzlibrary{patterns}

\usepackage{pgfplots}
\pgfmathdeclarefunction{gauss}{2}{%
	\pgfmathparse{1/(#2*sqrt(2*pi))*exp(-((x-#1)^2)/(2*#2^2))}%
}

\usepackage{color}